\documentclass{article} 
\usepackage{iclr2027_conference,times}

\usepackage{amsmath,amsfonts,bm}

\def\eqref#1{equation~\ref{#1}}

\def\1{\bm{1}}

\DeclareMathAlphabet{\mathsfit}{\encodingdefault}{\sfdefault}{m}{sl}
\SetMathAlphabet{\mathsfit}{bold}{\encodingdefault}{\sfdefault}{bx}{n}

\usepackage{hyperref}
\usepackage{url}
\usepackage{enumitem}
\usepackage{multirow}
\usepackage[dvipsnames]{xcolor}
\usepackage{listings}
\usepackage{pifont}
\usepackage{graphicx}
\usepackage{xcolor}
\usepackage{tabularx}
\usepackage[most]{tcolorbox}
\usepackage{amssymb}

\definecolor{casegreen}{RGB}{30,140,70}
\definecolor{casered}{RGB}{210,40,40}

\newcommand{\prmgreen}[1]{%
    \textcolor{casegreen}{\textbf{[BayesianPRM: #1]}}%
}

\newcommand{\prmred}[1]{%
    \textcolor{casered}{\textbf{[BayesianPRM: #1]}}%
}

\usepackage{enumitem}
\setlist[itemize]{leftmargin=*}

\usepackage{wrapfig}
\usepackage{caption}
\usepackage{booktabs}
\usepackage{graphicx}

\usepackage{subfigure}
\usepackage{booktabs}
\usepackage{multicol}
\usepackage[table]{xcolor}  
\hypersetup{
	colorlinks=true,
    allcolors={purple}
}
\usepackage{wrapfig}
\usepackage{amssymb}
\usepackage{bbm}
\usepackage{algorithmic}
\usepackage{algorithm}
\usepackage{amsmath,amsthm,amssymb}
\usepackage{graphicx}
\usepackage{amsmath}
\usepackage{mathtools}

\newtheorem{proposition}{Proposition}[section]

\usepackage{rotating}

\usepackage{booktabs}
\usepackage{multirow}
\usepackage[table]{xcolor}

\definecolor{lightpurple}{RGB}{245,241,252}
\definecolor{improvegray}{RGB}{130,130,130}

\title{Process Reward Modeling with Reliability- and Conservatism-Aware Bayesian Belief}

\author{
Zhongjian Qiao$^{1}$,
Zhixiang Zhou$^{2}$,
Jiafei Lyu$^{3}$,
Peisong Wang$^{4}$,
Qiaosheng Zhang$^{2}$, \\
~\textbf{Siyang Gao$^{1}$},
\textbf{Shuang Qiu$^{1}$}\thanks{Corresponding Author}\\
    $^{1}$CityUHK, $^{2}$Shanghai Innovation Institute, $^{3}$Tencent, \\ $^{4}$Institute of Automation,
Chinese Academy of Sciences \\
    ~\texttt{zhongqiao2-c@my.cityu.edu.hk}, \texttt{shuanqiu@cityu.edu.hk}
}

\iclrfinalcopy 
\begin{document}

\maketitle

\begin{abstract}
Process Reward Models (PRMs) play a central role in test-time scaling for large language model (LLM) reasoning. However, PRMs may provide inaccurate process rewards for test-time scaling, which we identify as arising from two sources of uncertainty: (1) \textit{PRM uncertainty} that arises from insufficient supervision data and imperfect PRM training; and (2) \textit{policy uncertainty} that results from the mismatch between the policy to collect supervision data and the policy for test-time scaling, leading to systematic reward overestimation. Previous studies fail to jointly consider both uncertainties. To bridge this gap, we propose BayesianPRM, which treats the process reward as a random variable within a reward hypothesis space and maintains a Bayesian belief over it. Since exact Bayesian inference over the continuous hypothesis space is generally intractable, we discretize the reward hypothesis space with a reward ensemble and introduce two complementary belief updates, each targeting one source of uncertainty: (1) to address PRM uncertainty, we introduce \textit{reliability-aware belief updating} that measures how well each ensemble member explains the supervision data and updates the belief with an amortized belief network optimized by the evidence lower bound (ELBO); and (2) to tackle policy uncertainty, we propose \textit{conservatism-aware belief calibration} which further reweights the belief toward lower rewards to suppress overestimation. These two updates yield a reliability- and conservatism-aware Bayesian belief, whose expectation determines the process reward. Experiments on multimodal reasoning benchmarks indicate that BayesianPRM consistently yields more accurate process reward estimates while enhancing test-time scaling performance.
\end{abstract}

\section{Introduction}

Test-time scaling effectively unlocks the reasoning potential of large language models (LLMs)~\citep{snell2024scaling,bi2024forest}, enabling smaller models to match or even surpass larger ones on challenging mathematics and coding tasks~\citep{liu2025can,kang2026t1}. Among common approaches such as Best-of-$N$~\citep{brown2024large}, Process Reward Models (PRMs)~\citep{lightman2024let,shao2024deepseekmath,li2026process} play a central role by providing step-level estimates of the probability of eventually reaching a correct answer. However, PRMs can produce inaccurate rewards, especially on complex and out-of-distribution (OOD) tasks~\citep{park2026know,li2026process}, suggesting that test-time scaling should explicitly account for \textit{uncertainty} in process rewards. However, previous studies have not comprehensively investigated such uncertainty.

To bridge this gap, we first analyze the uncertainty in process rewards for test-time scaling and categorize it into two distinct sources. The first is \textit{PRM uncertainty}, which arises from potentially insufficient supervision data and imperfect training procedures, limiting the generalization of PRMs to unseen problems; the second is \textit{policy uncertainty}, which arises from the mismatch between the policies used for PRM training and test-time scaling. For example, the Monte-Carlo labels used for PRM training are often generated by a strong policy, causing the PRM to learn the success probabilities associated with that strong policy. However, for test-time scaling, the policy is often weaker. Consequently, the process reward could be systematically overestimated. \citet{park2026know} and \citet{ma2026distributional} first observe the reward overestimation phenomenon, and employ quantile regression~\citep{koenker2001quantile} and conditional optimal transport~\citep{bunne2022supervised} on a calibration set to address it. We refer to the two approaches as QR-PRM and OT-PRM, respectively. They primarily address policy uncertainty, but do not explicitly account for PRM uncertainty. BetaPRM~\citep{li2026process} models the process reward with a Beta distribution, thereby capturing PRM uncertainty, while leaving policy uncertainty unaddressed. These raise a natural question: \textit{can we incorporate both PRM and policy uncertainty into process reward modeling for test-time scaling?}

Inspired by Bayesian learning~\citep{neal2012bayesian}, We propose BayesianPRM to jointly manage PRM and policy uncertainty. BayesianPRM treats the process reward as a random variable and maintains a belief distribution over the plausible reward hypothesis space. Since exact Bayesian inference over the continuous reward space is generally intractable, we discretize this space with a reward ensemble. We then introduce two stages of belief updates, each tackling one of the two uncertainties. First, to address PRM uncertainty, we introduce \textit{reliability-aware belief updating}. This stage measures how faithfully each ensemble member explains the supervision data. The update is performed with an amortized belief network trained to maximize an evidence lower bound (ELBO). Second, to tackle policy uncertainty, we propose \textit{conservatism-aware belief calibration}. By reweighting the belief toward lower reward values, this stage explicitly suppresses reward overestimation. These two updates yield a reliability-
and conservatism-aware Bayesian belief, whose expectation defines the final process reward. We compare BayesianPRM and relevant recent works in Table~\ref{tab:compare}. Our contributions are summarized as follows.
\begin{itemize}
    \item We identify the uncertainty of PRMs in test-time scaling and categorize it into two sources: PRM uncertainty and policy uncertainty.
    \item We propose BayesianPRM to address PRM and policy uncertainty through a two-stage belief updating framework: reliability-aware belief updating and conservatism-aware belief calibration.
    \item Experiments across multimodal reasoning benchmarks indicate that BayesianPRM consistently yields more accurate process reward estimates while enhancing test-time scaling performance.
\end{itemize}

\section{Preliminaries}
\label{sec:preliminary}
\begin{table}[]
    \centering
    \begingroup
    \setlength{\tabcolsep}{6pt}
    \caption{Comparison between BayesianPRM and relevant recent works.}
    \begin{tabular}{c|c|c|c}
    \toprule
       &  PRM Uncertainty & Policy Uncertainty & Technique \\
    \midrule
    BetaPRM~\citep{li2026process}  & {\color{green!70!black} \ding{51}} & {\color{red} \ding{55}} & Beta Distribution \\
    QR-PRM~\citep{park2026know} & {\color{red} \ding{55}} & {\color{green!70!black}\ding{51}}  & Quantile Regression \\
    OT-PRM~\citep{ma2026distributional} & {\color{red} \ding{55}} & {\color{green!70!black} \ding{51}}  & Optimal Transport \\
    \textbf{BayesianPRM (Ours)} & {\color{green!70!black} \ding{51}} & {\color{green!70!black} \ding{51}}  & Bayesian Learning \\

    \bottomrule
    \end{tabular}
    \label{tab:compare}
    \endgroup
\end{table}

\textbf{Process Reward Models.}
For an input problem $x$, let $s_{1:T}=(s_1,\ldots,s_T)$ denote a step-by-step solution, where $s_t$ is the $t$-th reasoning step. Given a prefix $c_t=(x,s_{\leq t})$ and policy $\pi$, define
$q_t^\pi=\mathrm{Pr}(\mathrm{final\,answer\,is\,correct}\mid c_t,\pi)$
as the probability that $\pi$ successfully completes $c_t$. A PRM parameterized by $\theta$ predicts $r_\theta(c_t)\in[0,1]$ to approximate $q_t^\pi$. Since $q_t^\pi$ is unobservable, PRMs typically learn from its Monte-Carlo estimate $\hat{q}^\pi_t$. For each prefix $c_t$, we sample $N_t$ continuations from $\pi$, among which $K_t$ are successful, yielding $\hat q_t^\pi=K_t/N_t$. Letting $\mathcal{D}_{\text{train}}=\{(c_t,\hat q_t^\pi)\}_{t=1}^{T}$ denote the training dataset, PRMs typically optimize the following cross-entropy loss:
\begin{equation}
\label{eq:cross_entropy}
    \mathcal{L}_\mathrm{CE}(\theta)
    =
    -\frac{1}{|\mathcal{D}_\text{train}|}
    \sum_{(c_t,\hat{q}_t^\pi)\in\mathcal{D}_\text{train}}
    \left[
    \hat{q}_t^\pi\log r_\theta(c_t)
    +(1-\hat{q}_t^\pi)\log(1-r_\theta(c_t))
    \right].
\end{equation}

\textbf{Bayesian Learning.} Bayesian learning~\citep{neal2012bayesian} provides a principled framework for uncertainty modeling. Given an input-label pair $(x,y)$ and a predictive model $p_\theta(y|x)$, it maintains a belief distribution over hypotheses $\theta\in\Theta$. Starting from a prior $p(\theta)$, the posterior after observing data $\mathcal D$ is obtained by Bayes' rule:
\begin{equation}
\label{eq:bayes}
p(\theta|\mathcal D)
=
\frac{p(\mathcal D|\theta)p(\theta)}
{\int_\Theta p(\mathcal D|\theta')p(\theta')d\theta'}
\propto
p(\mathcal D|\theta)p(\theta).
\end{equation}
The posterior represents uncertainty over plausible hypotheses, which could be incorporated into the final prediction by marginalization:
\begin{equation}
p(y|x,\mathcal D)
=
\int_\Theta p_\theta(y|x)p(\theta|\mathcal D)d\theta.
\end{equation}

\section{Analysis of Two Sources of Uncertainty}
\label{sec:illustration}

We first distinguish two sources of uncertainty in process reward estimation:
\textit{PRM uncertainty} and \textit{policy uncertainty}.
Let $\pi_1$ denote the Monte-Carlo rollout policy used to construct the PRM
supervision data, and $\pi_2$ the policy used for test-time scaling.
For a reasoning prefix $c_t$, the PRM prediction error under $\pi_2$ can be
decomposed as
\begin{equation}
    r_\theta(c_t)-q_t^{\pi_2}
    =
    \underbrace{r_\theta(c_t)-q_t^{\pi_1}}_{\text{PRM uncertainty}}
    +
    \underbrace{q_t^{\pi_1}-q_t^{\pi_2}}_{\text{policy uncertainty}},
    \qquad \forall\, c_t ,
\end{equation}
where $q_t^{\pi_1}$ and $q_t^{\pi_2}$ are the success probability defined in Section~\ref{sec:preliminary}. The first term characterizes the PRM prediction error under the same policy
used to construct its supervision data. Such errors may arise from insufficient
supervision and imperfect training, limiting the PRM's generalization to unseen
problems. We refer to this source as \textit{PRM uncertainty}, which is related to epistemic uncertainty.
The second term captures the change in the true continuation success probability
caused by the mismatch between $\pi_1$ and $\pi_2$, which we refer to as
\textit{policy uncertainty}.

\begin{figure}
    \centering
	\subfigure[] {\includegraphics[width=.24\linewidth]{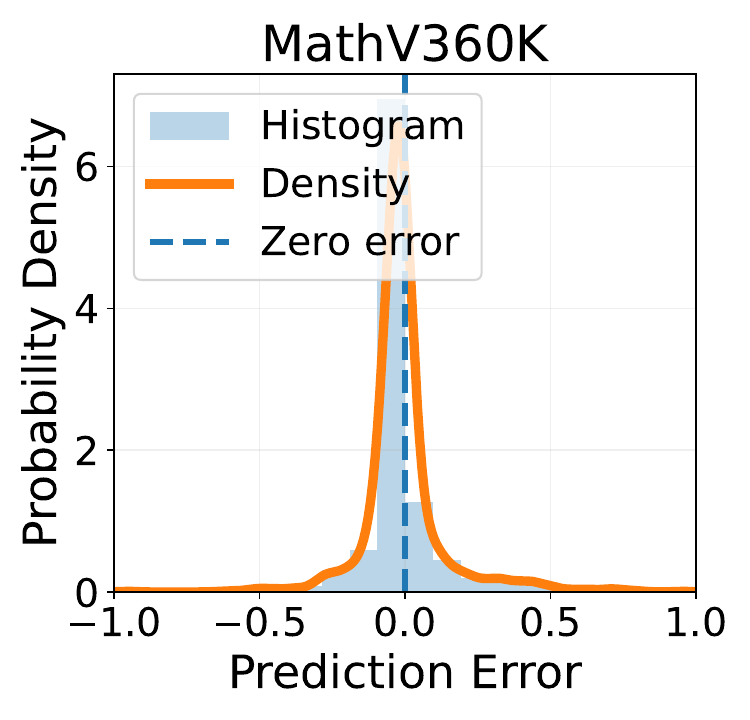}}
	\subfigure[] {\includegraphics[width=.24\linewidth]{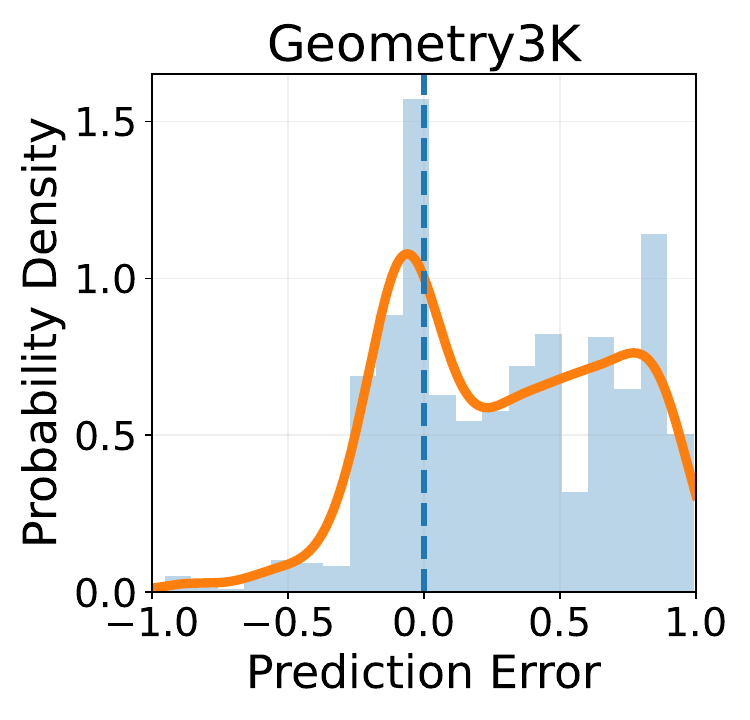}}
    \subfigure[] 
    {\includegraphics[width=.24\linewidth]{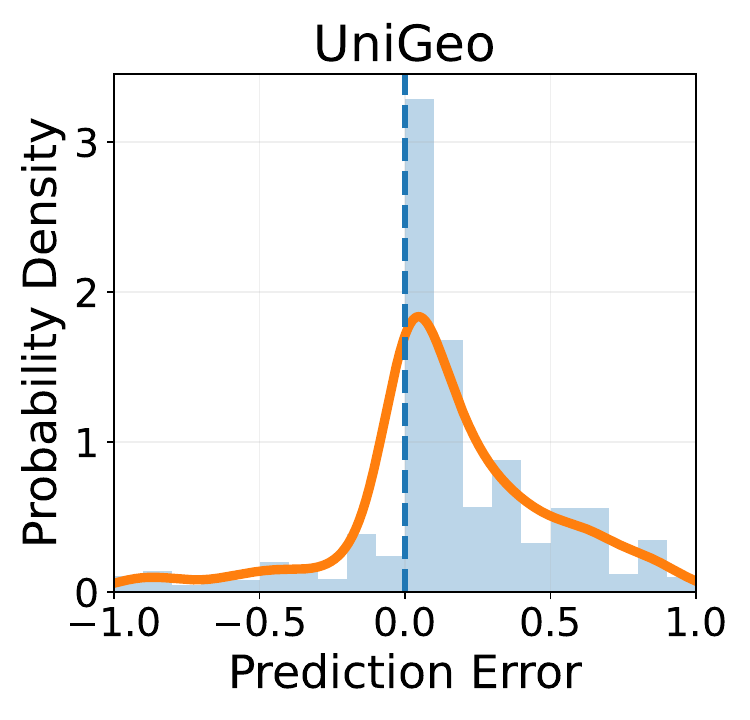}
         }
    \subfigure[] 
    {\includegraphics[width=.24\linewidth]{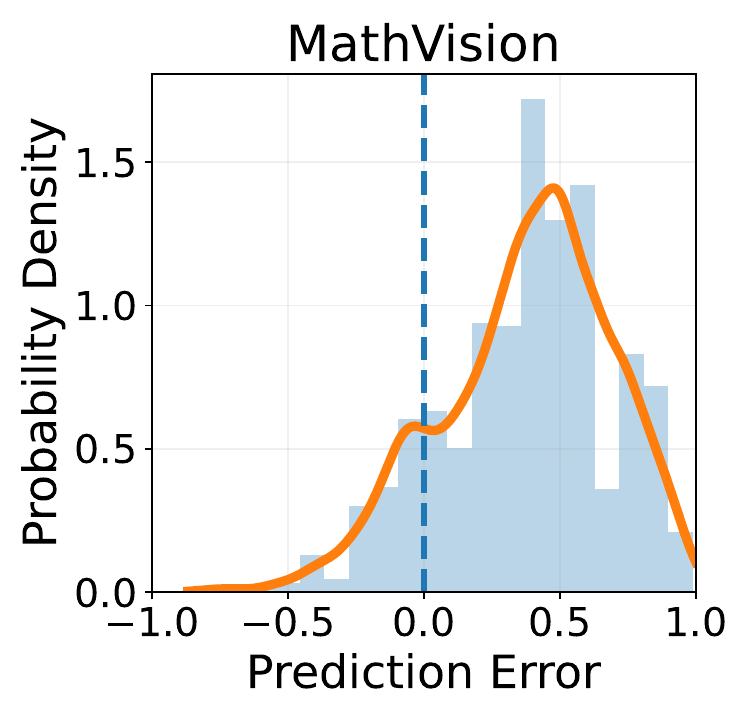}}
    \caption{\textbf{Illustration of PRM uncertainty and policy uncertainty.} \textbf{(a)-(b)} Prediction error distributions on MathV360K (ID) and Geometry3K (OOD), reflecting PRM uncertainty. \textbf{(c)-(d)} Prediction error distribution on UniGeo (ID) and MathVision (OOD), illustrating policy uncertainty. Prediction error is defined as the predicted reward minus the empirical success probability.}
    
    \label{fig:illustration}
\end{figure}

We next empirically validate the two sources of uncertainty under both in-distribution (ID) and out-of-distribution (OOD) settings. To examine PRM uncertainty, we use InternVL3-8B~\citep{zhu2025internvl3} as the base model and attach an MLP to its final layer as the PRM to predict a scalar process reward. We train the PRM on VisualPRM400K-v1.1~\citep{wang2026visualprm400k}, excluding the Geometry3K~\citep{lu2021inter} subset for subsequent OOD evaluation. For each remaining subset, 80\% of the data is used as the training dataset \(\mathcal{D}_{\text{train}}\) and the remaining 20\% as the validation dataset \(\mathcal{D}_{\text{val}}\). For the ID setting, we evaluate on \(\mathcal{D}_{\text{val}}\) of MathV360K ~\citep{shi2024math}, which can be regarded as independent and identical (i.i.d.) with \(\mathcal{D}_{\text{train}}\). For the OOD setting, we evaluate on the held-out Geometry3K subset. Therefore, policy uncertainty is removed in both settings. As shown in Figure~\ref{fig:illustration} (a) and (b), prediction errors on MathV360K are centered around zero but still exhibit variance, whereas the prediction errors on Geometry3K are substantially larger. This indicates that the PRM may remain unreliable on unseen problems, with the issue becoming more pronounced in OOD scenarios, reflecting PRM uncertainty.

We next examine policy uncertainty by using a weak rollout policy to
estimate the continuation success probability. We use
InternVL2.5-2B~\citep{chen2024expanding} as the rollout policy and
evaluate both an ID setting, $\mathcal{D}_\text{val}$ of UniGeo~\citep{chen2022unigeo}, and an OOD setting, MathVision~\citep{wang2024measuring}. For MathVision, we extract reasoning prefixes from
responses generated by the rollout policy. For each prefix in both
settings, we generate 16 continuations with the rollout policy and judge
their correctness against the ground-truth answer using
Qwen2.5-32B-Instruct~\citep{qwen2.5}. The resulting success rate is used as
the empirical success probability. As shown in Figure~\ref{fig:illustration} (c) and (d), the prediction error exhibits a positive bias under both ID and OOD settings, consistent with the observations reported in prior work~\citep{park2026know}, providing empirical evidence of policy uncertainty. The bias is more pronounced on MathVision, indicating that reward overestimation is more significant under the OOD setting.

\section{Methodology}

In this section, we formally present our BayesianPRM method. In Section~\ref{sec: prm_as_bayesian}, we formulate process reward prediction as Bayesian learning. In Section~\ref{sec:relia_conser}, we give a detailed demonstration of BayesianPRM. In Section~\ref{sec:continuous}, we extend the discussion to the continuous hypothesis space. In Section~\ref{sec:practical}, we present the practical implementation of BayesianPRM.

\subsection{Process Reward Modeling as Bayesian Learning}

\label{sec: prm_as_bayesian}

In this section, we formulate process reward modeling from a Bayesian learning perspective. Specifically, we treat the reward hypothesis as a random variable over a hypothesis space $\Theta$. Given a reasoning prefix $c_t$, each hypothesis $\theta \in \Theta$ predicts a process reward $r_\theta(c_t) = \mathbb{P}(Y_t=1 | c_t, \theta)$,
where $Y_t\in\{0,1\}$ indicates whether a continuation from $c_t$ eventually
reaches a correct answer. Let $\mathcal{D}_\text{train}=\{(c_t,K_t,N_t)\}_{t=1}^T$ denote the PRM training dataset. Let $p(\theta)$ be the prior belief over the hypothesis space $\Theta$. Given the observed supervision and the likelihood
$p(\mathcal{D}_\text{train}|\theta,c_t)$, the posterior belief $p(\theta|\mathcal{D}_\text{train},c_t)$ is obtained via Bayes' rule as in Equation~\ref{eq:bayes}. Given a reasoning prefix $c_t$, the posterior process reward prediction is obtained by marginalizing over plausible reward
hypotheses: $\tilde{r}(c_t)=\int_\Theta r_\theta(c_t)p(\theta|\mathcal{D}_\text{train},c_t)d\theta$.
Obtaining $\tilde{r}(c_t)$ requires posterior inference over the continuous hypothesis space $\Theta$, which is generally intractable. Therefore, we consider the discrete hypothesis space with a finite support set of hypotheses $\{\theta_m\}_{m=1}^M$. Specifically,
\begin{equation}
    \label{eq:discrete}
    p(\theta|\mathcal{D}_\text{train},c_t)=\sum_{m=1}^M \alpha_m(c_t)\delta(\theta-\theta_m), \quad \sum_{m=1}^M\alpha_m(c_t)=1, \quad \alpha_m(c_t)\geq0.
\end{equation}
Substituting Equation~\ref{eq:discrete} into $\tilde{r}(c_t)$, we have
$\tilde r(c_t)=\bm{\alpha}(c_t)^\top\bm r(c_t)$, where
$\bm{\alpha}(c_t)=[\alpha_1(c_t),\ldots,\alpha_M(c_t)]^\top$ and
$\bm r(c_t)=[r_1(c_t),\ldots,r_M(c_t)]^\top$.
Equation~\ref{eq:discrete} replaces the intractable integration with a
tractable weighted sum and naturally admits a reward ensemble
implementation~\citep{coste2024reward,duan2025efficient}, where each ensemble
member $r_m(\cdot)$ corresponds to a distinct reward hypothesis
$\theta_m\in\Theta$.

\subsection{Reliability- and Conservatism-aware Bayesian Belief}
\label{sec:relia_conser}

In this section, we present BayesianPRM, which incorporates PRM and policy uncertainty into process reward prediction by maintaining a prefix-conditioned Bayesian belief over ensemble reward hypotheses. BayesianPRM performs two successive belief updates: \textit{reliability-aware belief updating} accounts for PRM uncertainty by estimating hypothesis reliability from supervision data, while \textit{conservatism-aware belief calibration} addresses policy uncertainty by mitigating reward overestimation through conservative reweighting.

\textbf{Reliability-Aware Belief Updating.} Recall that in Section~\ref{sec: prm_as_bayesian}, the process reward is formulated as a weighted sum of reward hypotheses, $\tilde r(c_t)=\bm{\alpha}(c_t)^\top\bm{r}(c_t)$. The key challenge is to determine the belief weights $\bm{\alpha}(c_t)$.
A simple choice is a uniform prior,
$\bm{\alpha}^{\text{prior}}(c_t)=[\frac{1}{M},\ldots,\frac{1}{M}]^\top$.
However, this ignores PRM uncertainty, since different reward hypotheses may
have different reliability for the same prefix. We therefore introduce
\textit{reliability-aware belief updating}, which assigns greater belief mass
to more reliable reward hypotheses. For prefixes in the training dataset with Monte-Carlo labels, the posterior belief can be obtained
from Bayes' rule in the finite hypothesis space. However, such Monte-Carlo labels are unavailable for unseen prefixes at test time. We therefore adopt
amortized variational inference and train a belief network $w_\nu$ to predict
the prefix-conditioned posterior belief. Specifically, for a prefix $c_t$, the belief network $w_\nu$ outputs a reliability-aware belief vector $\bm{\alpha}^\text{rel}(c_t)=\left[
\alpha^{\text{rel}}_1(c_t),
\ldots,
\alpha^{\text{rel}}_M(c_t)
\right]^\top$, where $\sum_{i=m}^M\alpha^\text{rel}_m(c_t)=1$ and $\alpha^\text{rel}_m(c_t)\geq0$. Intuitively, we expect $w_\nu$ to predict the reliability of each reward ensemble member for the given prefix. Taking $\bm{\alpha}^\text{prior}(c_t)$ as the prior belief, we further define the reliability likelihood function to perform amortized variational inference. Specifically, we use a validation set $\mathcal{D}_\text{val}=\{(c_t,K_t,N_t)\}_{t=1}^{T^\prime}$, which is independent of the PRM training set $\mathcal{D}_\text{train}$. Since $r_m(c_t)$ represents the predicted success probability under $\theta_m$, $K_t$ follows a binomial distribution: $(K_t \mid N_t, c_t, \theta_m)
\sim
\operatorname{Binomial}\!\left(N_t, r_m(c_t)\right)$. Therefore, the reliability likelihood of the supervision data in $\mathcal{D}_\text{val}$ under $\theta_m$ is: $p(K_t \mid N_t, c_t, \theta_m)
=
\binom{N_t}{K_t}
r_m(c_t)^{K_t}
\left(1-r_m(c_t)\right)^{N_t-K_t}.$ This likelihood measures how well $r_m(c_t)$ explains the supervision data, and thus serves as a measure of its reliability for prefix $c_t$. Denote $\ell_{t,m}\triangleq\log p(K_t|N_t,c_t,\theta_m)=
\log \binom{N_t}{K_t}
+
K_t \log r_m(c_t)
+
(N_t-K_t)\log\left(1-r_m(c_t)\right)$ as the log-likelihood, we perform amortized variational inference, training the belief network on $\mathcal{D}_\text{val}$ by minimizing the following negative evidence lower bound (ELBO):
\begin{equation}
\label{eq:elbo}
    \begin{aligned}
    &\mathcal{L}_{\text{rel}}(\nu)=-\frac{1}{T^\prime}\sum_{t=1}^{T^\prime}\left[\sum_{m=1}^M\alpha_m^{\text{rel}}(c_t)\ell_{t,m}-\beta_1\cdot D_\mathrm{KL}(\bm{\alpha}^\text{rel}(c_t)\Vert\bm{\alpha}^\text{prior}(c_t))\right]\\
    &\doteq -\frac{1}{T^\prime}\sum_{t=1}^{T^\prime}\sum_{m=1}^M\alpha^\text{rel}_m(c_t)\left(K_t\log r_m(c_t)+(N_t-K_t)\log(1-r_m(c_t))-\beta_1\log(M\alpha_m^\text{rel}(c_t))\right),
\end{aligned}
\end{equation}
where $\bm{\alpha}^\text{prior}(c_t)=[\frac{1}{M},...,\frac{1}{M}]^\top$ is the uniform prior belief, and $\beta_1$ is the temperature coefficient. Note that the constant term $\log\binom{N_t}{K_t}$ is omitted in Equation~\ref{eq:elbo} since it does not affect optimization. To further justify the objective in Equation~\ref{eq:elbo}, we characterize its optimal solution and establish its connection to Bayesian inference in Proposition~\ref{prop:bayes_optimal}.
\begin{proposition} 
\label{prop:bayes_optimal}
The optimal solution of Equation~\ref{eq:elbo} recovers the following Bayesian posterior:
\begin{equation}
    \label{eq:optimal_rel}{\alpha_m^{\mathrm{rel}}(c_t)}=\frac{\alpha^\mathrm{prior}_m(c_t)\exp\left(\ell_{t,m}/\beta_1\right)}{\sum_{j=1}^M\alpha^\mathrm{prior}_j(c_t)\exp\left(\ell_{t,j}/\beta_1\right)}\propto{\alpha^\mathrm{prior}_m(c_t)}\cdot {p(K_t|N_t,c_t,\theta_m)^{1/\beta_1}},\quad m=1,2,...,M
\end{equation}
where $\alpha^\mathrm{prior}_m(c_t)$ is the $m$-th element of $\bm{\alpha}^\mathrm{prior}(c_t)$. Under a uniform belief $\bm{\alpha}^\mathrm{prior}(c_t)$, we have $\alpha_m^{\mathrm{rel}}(c_t)\propto p(K_t|N_t,c_t,\theta_m)^{1/\beta_1}$.  
\end{proposition}
\textbf{Remark.} Proposition~\ref{prop:bayes_optimal} provides a Bayesian interpretation of Equation~\ref{eq:elbo}. Specifically, optimizing Equation~\ref{eq:elbo} corresponds to updating the prior belief $\bm{\alpha}^\text{prior}(c_t)$ according to the reliability likelihood of the supervision data in $\mathcal{D}_\text{val}$, thus addressing PRM uncertainty. $\beta_1$ controls the strength of this belief update: when $\beta_1=1$, the solution recovers the standard Bayesian posterior, while $\beta_1>1$ yields a flatter tempered posterior and $\beta_1<1$ yields a sharper one. This update directly follows Bayes' rule in Equation~\ref{eq:bayes}. The key difference is that, the belief network amortizes this Bayesian update, enabling the reliability-aware belief to generalize to unseen prefixes at test time.

\textbf{Conservatism-Aware Belief Calibration.} Although the reliability-aware belief addresses PRM uncertainty, reward overestimation may still occur since policy uncertainty is not accounted for. Therefore, we introduce a second stage of belief updating, termed \textit{conservatism-aware belief calibration}. In this stage, we take the learned reliability-aware belief as the prior belief and further calibrate it by shifting belief mass toward reward hypotheses with lower reward predictions.

Specifically, we treat the reliability-aware belief $\bm{\alpha}^{\text{rel}}(c_t)$ as the prior belief,  and define the hybrid Bayesian belief as
\begin{equation}
\label{eq:hybrid_belief_objective}
\bm{\alpha}^{\text{hyb}}(c_t)
=
\arg\min_{\bm{\alpha}\in\Delta^{M-1}}
\left\{
\bm{\alpha}^{\top}\bm r(c_t)
+
\beta_2\cdot
D_{\mathrm{KL}}
\left(
\bm{\alpha}
\Vert
\bm{\alpha}^{\text{rel}}(c_t)
\right)
\right\},
\end{equation}
where $\Delta^{M-1}=\{\bm{\alpha}\in\mathbb{R}_{\geq 0}^{M}:\sum_{m=1}^{M}\alpha_m=1\}$ denotes the probability simplex, and $\beta_2>0$ is a temperature coefficient. Solving Equation~\ref{eq:hybrid_belief_objective} yields the following closed-form solution:
\begin{equation}
\label{eq:hybrid_belief}
{\alpha_m^{\text{hyb}}(c_t)}
=
\frac{
\alpha_m^{\text{rel}}(c_t)
\exp\left(
-
r_m(c_t)/
\beta_2
\right)
}{
\sum_{j=1}^{M}
\alpha_j^{\text{rel}}(c_t)
\exp\left(
-
r_j(c_t)/
\beta_2
\right)
}\propto{\alpha_m^{\text{rel}}(c_t)}
\cdot{\exp\left(
-
r_m(c_t)/
\beta_2
\right)},\quad m=1,2,...,M.
\end{equation}
\textbf{Remark.} Equation~\ref{eq:hybrid_belief} admits a Bayesian interpretation: $\bm{\alpha}^\text{hyb}(c_t)$ is the Bayesian posterior obtained by taking $\bm{\alpha}^\text{rel}(c_t)$ as the prior and $\exp(-\bm{r}(c_t)/\beta_2)$ as the likelihood. This formulation is closely related to the control-as-inference framework~\citep{levine2018reinforcement}, where an exponentially weighted reward is used as a likelihood. The key difference is that control-as-inference typically assigns larger likelihood to higher-reward outcomes, whereas Equation~\ref{eq:hybrid_belief} introduces conservatism by assigning larger likelihood to lower-reward predictions, thereby mitigating reward overestimation. Note that, unlike reliability-aware belief updating, this stage does not require training a belief network. At test time, the ensemble reward predictions $\bm{r}(c_t)$ are available. Therefore, $\bm{\alpha}^\text{hyb}(c_t)$ can be directly obtained by applying exponential reweighting in Equation~\ref{eq:hybrid_belief} to $\bm{\alpha}^\text{rel}(c_t)$.

Substituting Equation~\ref{eq:optimal_rel} into Equation~\ref{eq:hybrid_belief}, we obtain:
\begin{equation}
    \label{eq:final_belief}
    \underbrace{\alpha_m^{\text{hyb}}(c_t)}_{\text{posterior}}\propto\underbrace{\alpha_m^{\text{prior}}(c_t)}_{\text{prior}}\cdot\underbrace{\exp\left(\ell_{t,m}/\beta_1-r_m(c_t)/
\beta_2\right)}_{\text{likelihood}},\quad m=1,2,...,M.
\end{equation}
In Equation~\ref{eq:final_belief}, the likelihood consists of two components: the first term, $\exp(\ell_{t,m}/\beta_1)$, accounts for the reliability of each reward hypothesis, while the second term, $\exp(-r_m(c_t)/\beta_2)$, introduces conservatism by favoring hypotheses with lower reward predictions.  $\beta_1$ and $\beta_2$ control the strengths of the reliability-aware and conservatism-aware belief updates, respectively. Therefore, $\bm{\alpha}^\text{hyb}(c_t)$ represents a reliability- and conservatism-aware Bayesian belief. The final process reward is computed as: $\tilde{r}(c_t)
=
\left[
\bm{\alpha}^{\text{hyb}}(c_t)
\right]^{\top}
\bm r(c_t).$ By aggregating the ensemble predictions under $\bm{\alpha}^\text{hyb}(c_t)$, $\tilde{r}(c_t)$ jointly accounts for reliability and conservatism. Specifically, it assigns larger weights to reward hypotheses that better explain the supervision data while favoring more conservative hypotheses with lower reward predictions, thereby addressing both PRM uncertainty and policy uncertainty.

\subsection{Connection with Continuous Hypothesis Space} 
\label{sec:continuous}

Thus far, we have considered the discrete hypothesis space represented by a finite reward ensemble for computational tractability. In this part, we further show that this finite-ensemble formulation can be interpreted as an implicit Bayesian update over an underlying continuous reward hypothesis space $\Theta$. Specifically, let $p(\theta)$ denote the prior belief over $\Theta$, and we assume the ensemble hypotheses $\theta_1,...,\theta_M$ are sampled independently and identically from $p(\theta)$, that is, $\theta_1,...,\theta_M\overset{\mathrm{i.i.d.}}{\sim} p(\theta)$. Then $\tilde{r}(c_t)=\mathbb{E}_{\theta_{1:M}\overset{\mathrm{i.i.d.}}\sim p(\theta)}\sum_{m=1}^M\left[\alpha^\text{hyb}_{\theta_m}(c_t)\cdot r_{\theta_m}(c_t)\right]$. For consistency with the preceding notations, we abbreviate $\alpha^\text{hyb}_{\theta_m}(c_t)$ as $\alpha^\text{hyb}_{m}(c_t)$ and $r_{\theta_m}(c_t)$ as $r_m(c_t)$. Our goal is to construct a posterior belief $\tilde{p}(\theta|c_t)$ such that $\tilde{r}(c_t)=\mathbb{E}_{\theta_{1:M}\overset{\mathrm{i.i.d.}}\sim p(\theta)}\sum_{m=1}^M\left[\alpha^\text{hyb}_{m}(c_t)\cdot r_{m}(c_t)\right]=\mathbb{E}_{\theta\sim\tilde{p}(\theta|c_t)}\left[r_\theta(c_t)\right]$. In other words, we aim to show that the expected weighted sum over prior hypotheses $p(\theta)$ can be equivalently represented as an expectation under a continuous posterior belief $\tilde{p}(\theta|c_t)$. According to Bayes' rule, $\tilde{p}(\theta|c_t)$ could be obtained by reweighting $p(\theta)$ with a likelihood ratio $\xi(\theta;c_t)$, that is, $\tilde{p}(\theta|c_t)=p(\theta)\xi(\theta;c_t)$. Denote $s(\theta;c_t)\triangleq\ell_{t,\theta}/\beta_1-r_\theta(c_t)/\beta_2$, where $\ell_{t,\theta}\triangleq\log p(K_t|N_t,c_t,\theta)=\log \binom{N_t}{K_t}
+
K_t \log r_\theta(c_t)
+
(N_t-K_t)\log\left(1-r_\theta(c_t)\right)$. Then we characterize $\xi(\theta;c_t)$ as follows.

\begin{proposition}[Likelihood Ratio]
\label{prop:likelihood}Let $\mathcal{F}$ denote the cumulative distribution function of $s(\theta;c_t)$ with $\theta\sim p(\theta)$. Then the likelihood ratio $\xi(\theta;c_t)$ is given by:
\begin{equation}
    \label{eq:xi}\xi(\theta;c_t)=\sum_{m=1}^M\bar{\alpha}_{m}(\theta;c_t)\frac{M!}{(m-1)!(M-m)!}\mathcal{U}_m(\mathcal{F}(s(\theta;c_t))),
\end{equation}
where $\mathcal{U}_m(\mathcal{F}(s(\theta;c_t)))=[\mathcal{F}(s(\theta;c_t))]^{m-1}[1-\mathcal{F}(s(\theta;c_t))]^{M-m}$, and $\{\bar{\alpha}_m(\theta;c_t)\}_{m=1}^M$ denotes the weighting coefficients defined in Appendix~\ref{app:proof_xi}.

\end{proposition}

\textbf{Remark.} For each $m\in\{1,...,M\}$, $\mathcal{U}_m(\mathcal{F}(s(\theta;c_t))$ is a unimodal continuous function. This function attains its maximum value at the $\frac{m-1}{M-1}$-quantile of $s(\theta;c_t)$, that is, $\mathcal{F}(s(\theta;c_t))=\frac{m-1}{M-1}$. Therefore, each term in Equation~\ref{eq:xi} corresponds to a continuous component concentrated around a particular quantile of $s(\theta;c_t)$.
When the prior $p(\theta)$ is uniform, the posterior satisfies $\tilde{p}(\theta|c_t)\propto\xi(\theta;c_t)$, and is a weighted mixture of these continuous components. Compared with the discrete belief in Equation~\ref{eq:discrete}, $\tilde{p}(\theta|c_t)$ distributes the belief weights over the continuous hypothesis space $\Theta$. This shows that our finite-ensemble formulation can be viewed as an implicit Bayesian update that reshapes the prior belief into a continuous posterior through likelihood reweighting. 

\subsection{Practical Implementation}
\label{sec:practical}
In this section, we present the practical implementation of BayesianPRM, including the model architecture and the training and inference procedures.

\textbf{Model Architecture.} BayesianPRM consists of a reward ensemble and a belief network built on a shared LLM backbone. Given a prefix $c_t$, we insert a marker token \texttt{<prm>}, and the LLM backbone produces a final-layer hidden representation $z_t$ at \texttt{<prm>}. We attach $M$ reward heads $\{r_m\}_{m=1}^M$ to $z_t$. To prevent ensemble collapse and maintain diversity among the reward predictions, following~\cite{lee2026inference}, we pair the reward heads with randomly initialized and fixed prior networks $\{f_m\}_{m=1}^M$~\citep{osband2018randomized,osband2023epistemic}. The prediction of the $m$-th process reward is given by $r_m(c_t)=\sigma(g_m(z_t)+\lambda_\text{prior}f_m(z_t))$, where $g_m(z_t)$ is the reward logit of $r_m$, $\lambda_\text{prior}$ is a hyperparameter and $\sigma(\cdot)$ is the sigmoid function. For the belief network $w_\nu$, it takes $x_t=[z_t;\bm{r}(c_t)]$ as the input and outputs $M$ logits, the reliability-aware belief is computed as $\bm{\alpha}^\text{rel}(c_t)=\mathrm{Softmax}(w_\nu(x_t))$.

\textbf{Training Procedure.} BayesianPRM includes a two-stage training procedure. At the first stage, we train the reward ensemble on the training set $\mathcal{D}_{\text{train}}$. Specifically, the LLM backbone and each reward head are optimized with the binary cross-entropy loss in Equation~\ref{eq:cross_entropy}. At the second stage, we freeze the LLM backbone and all reward heads and train only the belief network on the validation set $\mathcal{D}_{\text{val}}$ by optimizing the ELBO in Equation~\ref{eq:elbo}. We use $\mathcal{D}_{\text{val}}$ instead of $\mathcal{D}_{\text{train}}$ for belief network training to prevent it from overfitting to the training data used to optimize the reward ensemble.

\textbf{Inference Procedure.} During inference, given a prefix $c_t$, the reward ensemble first produces the reward predictions $\bm r(c_t)$, while the belief network outputs the reliability-aware belief $\bm{\alpha}^{\text{rel}}(c_t)$. We then perform conservatism-aware belief calibration to obtain $\bm{\alpha}^\text{hyb}(c_t)$ by Equation~\ref{eq:hybrid_belief}. The final process reward is computed as $\tilde{r}(c_t)
=
\left[
\bm{\alpha}^{\text{hyb}}(c_t)
\right]^{\top}
\bm r(c_t).$ We present more implementation details and the pseudo-code of BayesianPRM in Appendix~\ref{sec:app_details}.

\begin{table*}[t]
\centering
\caption{
\textbf{PRM accuracy evaluation results.} We evaluate Brier Score, Positive Brier Score, ECE, and AverageCE on four benchmarks: MathVision, the Vision Only subset of MathVerse (denoted as MathVerse-VO), MathVista, and MMStar. {\color{improvegray}$\downarrow$} indicates that lower values are better. $\pm$ denotes the standard deviation across three random training seeds.
}
\resizebox{\textwidth}{!}{
\begin{tabular}{lccccc}
\toprule

& \textbf{MathVision}
& \textbf{MathVerse-VO}
& \textbf{MathVista}
& \textbf{MMStar}
& \textbf{Average} \\
\midrule

\multicolumn{6}{c}{\textit{\textbf{Brier Score}} {\color{improvegray}$\downarrow$}} \\
\midrule

Standard PRM
& 0.2025{\small\textcolor{improvegray}{$\pm 0.0194$}}
& 0.2274{\small\textcolor{improvegray}{$\pm 0.0143$}}
& 0.1704{\small\textcolor{improvegray}{$\pm 0.0082$}}
& 0.2357{\small\textcolor{improvegray}{$\pm 0.0236$}}
& 0.2090\\

BetaPRM
& 0.2377{\small\textcolor{improvegray}{$\pm 0.0221$}}
& 0.2489{\small\textcolor{improvegray}{$\pm 0.0163$}}
& 0.1712{\small\textcolor{improvegray}{$\pm 0.0148$}}
& 0.2427{\small\textcolor{improvegray}{$\pm 0.0121$}}
& 0.2251 \\

DreamPRM
& 0.2461{\small\textcolor{improvegray}{$\pm 0.0292$}}
& 0.2498{\small\textcolor{improvegray}{$\pm 0.0180$}}
& 0.1904{\small\textcolor{improvegray}{$\pm 0.0064$}}
& 0.2342{\small\textcolor{improvegray}{$\pm 0.0208$}}
& 0.2301\\

\rowcolor{lightpurple}
BayesianPRM
& \textbf{0.1411}{\small\textcolor{improvegray}{$\pm 0.0201$}}
& \textbf{0.1848}{\small\textcolor{improvegray}{$\pm 0.0112$}}
& \textbf{0.1579}{\small\textcolor{improvegray}{$\pm 0.0107$}}
& \textbf{0.2085}{\small\textcolor{improvegray}{$\pm 0.0153$}}
& \textbf{0.1731}\\

\midrule
\multicolumn{6}{c}{\textit{\textbf{Positive Brier Score}}
{\color{improvegray}$\downarrow$}} \\
\midrule

Standard PRM
& 0.1931{\small\textcolor{improvegray}{$\pm 0.0314$}}
& 0.2147{\small\textcolor{improvegray}{$\pm 0.0203$}}
& 0.1403{\small\textcolor{improvegray}{$\pm 0.0149$}}
& 0.2300{\small\textcolor{improvegray}{$\pm 0.0182$}}
& 0.1945 \\

BetaPRM
& 0.2308{\small\textcolor{improvegray}{$\pm 0.0226$}}
& 0.2413{\small\textcolor{improvegray}{$\pm 0.0209$}}
& 0.1496{\small\textcolor{improvegray}{$\pm 0.0157$}}
& 0.2371{\small\textcolor{improvegray}{$\pm 0.0261$}}
& 0.2147 \\

DreamPRM
& 0.2428{\small\textcolor{improvegray}{$\pm 0.0251$}}
& 0.2445{\small\textcolor{improvegray}{$\pm 0.0167$}}
& 0.1808{\small\textcolor{improvegray}{$\pm 0.0189$}}
& 0.2276{\small\textcolor{improvegray}{$\pm 0.0203$}}
& 0.2239\\

\rowcolor{lightpurple}
BayesianPRM
& \textbf{0.1207}{\small\textcolor{improvegray}{$\pm 0.0085$}}
& \textbf{0.1602}{\small\textcolor{improvegray}{$\pm 0.0113$}}
& \textbf{0.1130}{\small\textcolor{improvegray}{$\pm 0.0076$}}
& \textbf{0.1965}{\small\textcolor{improvegray}{$\pm 0.0128$}}
& \textbf{0.1476}\\






\midrule
\multicolumn{6}{c}{\textit{\textbf{ECE}}
{\color{improvegray}$\downarrow$}} \\
\midrule

Standard PRM
& 0.3213{\small\textcolor{improvegray}{$\pm 0.0263$}}
& 0.2868{\small\textcolor{improvegray}{$\pm 0.0196$}}
& 0.1481{\small\textcolor{improvegray}{$\pm 0.0075$}}
& 0.2704{\small\textcolor{improvegray}{$\pm 0.0351$}}
& 0.2567\\

BetaPRM
& 0.3777{\small\textcolor{improvegray}{$\pm 0.0411$}}
& 0.3241{\small\textcolor{improvegray}{$\pm 0.0259$}}
& 0.1654{\small\textcolor{improvegray}{$\pm 0.0107$}}
& 0.2760{\small\textcolor{improvegray}{$\pm 0.0212$}}
& 0.2858 \\

DreamPRM
& 0.3585{\small\textcolor{improvegray}{$\pm 0.0399$}}
& 0.3139{\small\textcolor{improvegray}{$\pm 0.0240$}}
& 0.1902{\small\textcolor{improvegray}{$\pm 0.0137$}}
& 0.2827{\small\textcolor{improvegray}{$\pm 0.0205$}}
& 0.2863\\

\rowcolor{lightpurple}
BayesianPRM
& \textbf{0.2102}{\small\textcolor{improvegray}{$\pm 0.0116$}}
& \textbf{0.2016}{\small\textcolor{improvegray}{$\pm 0.0154$}}
& \textbf{0.1108}{\small\textcolor{improvegray}{$\pm 0.0072$}}
& \textbf{0.2200}{\small\textcolor{improvegray}{$\pm 0.0108$}}
& \textbf{0.1857}\\
\midrule
\multicolumn{6}{c}{\textit{\textbf{AverageCE}}
{\color{improvegray}$\downarrow$}} \\
\midrule

Standard PRM
& 0.2849{\small\textcolor{improvegray}{$\pm 0.0266$}}
& 0.2140{\small\textcolor{improvegray}{$\pm 0.0160$}}
& 0.1139{\small\textcolor{improvegray}{$\pm 0.0100$}}
& 0.2159{\small\textcolor{improvegray}{$\pm 0.0159$}}
& 0.2072\\

BetaPRM
& 0.2823{\small\textcolor{improvegray}{$\pm 0.0126$}}
& 0.2401{\small\textcolor{improvegray}{$\pm 0.0228$}}
& 0.1372{\small\textcolor{improvegray}{$\pm 0.0060$}}
& 0.1901{\small\textcolor{improvegray}{$\pm 0.0084$}}
& 0.2124 \\

DreamPRM
& 0.2973{\small\textcolor{improvegray}{$\pm 0.0204$}}
& 0.2305{\small\textcolor{improvegray}{$\pm 0.0172$}}
& 0.1622{\small\textcolor{improvegray}{$\pm 0.0120$}}
& 0.2384{\small\textcolor{improvegray}{$\pm 0.0097$}}
& 0.2321\\

\rowcolor{lightpurple}
BayesianPRM
& \textbf{0.2301}{\small\textcolor{improvegray}{$\pm 0.0193$}}
& \textbf{0.1733}{\small\textcolor{improvegray}{$\pm 0.0092$}}
& \textbf{0.0946}{\small\textcolor{improvegray}{$\pm 0.0013$}}
& \textbf{0.1720}{\small\textcolor{improvegray}{$\pm 0.0029$}}
& \textbf{0.1675}\\
\bottomrule
\end{tabular}
}
\label{tab:calibration}
\end{table*}

\section{Experiments}
\label{sec:experiments}
In this section, we empirically evaluate BayesianPRM from two perspectives. First, we investigate whether BayesianPRM improves the accuracy of process reward prediction. Second, we examine whether BayesianPRM can enhance test-time scaling performance.

\subsection{Process Reward Accuracy Evaluation}
\label{sec:accuracy}
\textbf{Datasets.} We mainly evaluate our method on multimodal reasoning benchmarks. The dataset we use during training is VisualPRM400K-v1.1~\citep{wang2026visualprm400k}. For BayesianPRM, we randomly sample 80\% of the data as $\mathcal{D}_\text{train}$ and use the remaining 20\% as $\mathcal{D}_\text{val}$ for belief network training. For the other baselines, we use the full dataset for training. For evaluation benchmarks, we use MathVision~\citep{wang2024math}, the Vision Only subset of MathVerse~\citep{zhang2024mathverse} (denoted as MathVerse-VO), MathVista~\citep{lu2024mathvista}, and MMStar~\citep{chen2024we}. For each image-question pair, we first use InternVL3-8B to generate a reasoning response, from which intermediate reasoning prefixes are extracted. Starting from each prefix $c_t$, we perform 16 Monte-Carlo rollouts using InternVL3-8B. Each rollout is then evaluated by Qwen2.5-32B-Instruct as the judge model to determine its correctness. The proportion of correct rollouts is treated as the empirical success probability $\hat{p}_t$. The $(c_t,\hat{p}_t)$ pairs constitute $\mathcal{D}_\text{test}$ for evaluation.

\textbf{Evaluation Metrics.} Following~\cite{park2026know}, we assess the accuracy of PRM predictions using the following calibration metrics: Brier Score~\citep{glenn1950verification}, Positive Brier Score~\citep{park2026know}, Expected Calibration Error (ECE)~\citep{naeini2015obtaining} and Average Calibration Error (AverageCE)~\citep{neumann2018relaxed} between the empirical success probability and the PRM predictions. The computation of these metrics is presented in Appendix~\ref{sec:app_metrics}.

\textbf{Baselines.} We consider the following baselines for comparison with BayesianPRM: (1) Standard PRM, which is optimized using the standard cross-entropy loss; (2) BetaPRM~\citep{li2026process}, which models the process reward as a Beta distribution; and (3) DreamPRM~\citep{cao2026dreamprm}, which reweights different domains of the dataset for multimodal reasoning. All PRMs are built on InternVL3-8B. Since QR-PRM~\citep{park2026know} requires an additional calibration set, we provide a separate comparison with it in Appendix~\ref{sec:app_calibration_sets}.

\textbf{Results.} As shown in Table~\ref{tab:calibration}, BayesianPRM consistently achieves the lowest calibration error across all four benchmarks and all four metrics. Specifically, BayesianPRM reduces the average Brier Score, Positive Brier Score, ECE, and AverageCE by \textbf{17.2\%}, \textbf{24.1\%}, \textbf{27.7\%}, and \textbf{19.2\%}, respectively, compared with Standard PRM. In particular, BayesianPRM achieves a substantially lower Positive Brier Score, suggesting that conservatism-aware belief calibration effectively reduces reward overestimation. Meanwhile, BayesianPRM also consistently improves calibration on benchmarks where Standard PRM already exhibits relatively small prediction errors, such as MathVista, suggesting that reliability-aware belief updating provides benefits beyond simply mitigating systematic overestimation.

\begin{table*}[t]
\centering
\caption{
\textbf{PRM-guided test-time scaling performance comparison.} We compare Best-of-\(N\) (\(N=16\)) and Instance-Adaptive Scaling (IAS, \(N_{\max}=16\)) on MathVision, MathVerse-VO, MathVista and MMStar.  {\color{improvegray}$\uparrow$} indicates that higher values are better. $\pm$ denotes the standard deviation across three random training seeds.
}
\resizebox{\textwidth}{!}{
\begin{tabular}{lccccc}
\toprule

& \textbf{MathVision}
& \textbf{MathVerse-VO}
& \textbf{MathVista}
& \textbf{MMStar}
& \textbf{Average} \\
\midrule

InternVL3-8B
& 29.93
& 45.05
& 65.90
& 66.27
& 51.79 \\

\midrule
\multicolumn{6}{c}{\textit{\textbf{Best-of-N ($N=16$)}}{\color{improvegray}$\uparrow$}} \\
\midrule

Self-Consistency
& 33.22
& 47.46
& 68.00
& 68.60
& 54.32 \\

Standard PRM
& 32.57{\small\textcolor{improvegray}{$\pm 0.78$}}
& 50.78{\small\textcolor{improvegray}{$\pm 1.17$}}
& 69.00{\small\textcolor{improvegray}{$\pm 1.10$}}
& 67.47{\small\textcolor{improvegray}{$\pm 0.61$}}
& 54.96 \\

BetaPRM
& 34.31{\small\textcolor{improvegray}{$\pm 0.70$}}
& \textbf{51.85}{\small\textcolor{improvegray}{$\pm 0.63$}}
& 68.66{\small\textcolor{improvegray}{$\pm 1.15$}}
& 67.77{\small\textcolor{improvegray}{$\pm 1.54$}}
& 55.65 \\

DreamPRM
& 34.79{\small\textcolor{improvegray}{$\pm 1.01$}}
& 51.22{\small\textcolor{improvegray}{$\pm 0.64$}}
& 69.13{\small\textcolor{improvegray}{$\pm 0.65$}}
& 67.18{\small\textcolor{improvegray}{$\pm 1.12$}}
& 55.58\\

\rowcolor{lightpurple}
BayesianPRM
& \textbf{35.86}{\small\textcolor{improvegray}{$\pm 0.70$}}
& 51.40{\small\textcolor{improvegray}{$\pm 1.11$}}
& \textbf{69.40}{\small\textcolor{improvegray}{$\pm 0.53$}}
& \textbf{69.33}{\small\textcolor{improvegray}{$\pm 0.60$}}
& \textbf{56.50}\\

\midrule
\multicolumn{6}{c}{
    \textbf{\textit{IAS ($N_{\max}=16$)}}
    {\color{improvegray}{$\uparrow$}}
} \\
\midrule

Standard PRM
& 32.19{\small\textcolor{improvegray}{$\pm 0.68$}}
& 50.46{\small\textcolor{improvegray}{$\pm 0.61$}}
& \textbf{69.11}{\small\textcolor{improvegray}{$\pm 0.98$}}
& 67.96{\small\textcolor{improvegray}{$\pm 0.44$}}
& 54.93\\

BetaPRM
& 34.26{\small\textcolor{improvegray}{$\pm 0.70$}}
& 50.58{\small\textcolor{improvegray}{$\pm 1.06$}}
& 68.66{\small\textcolor{improvegray}{$\pm 1.15$}}
& 68.18{\small\textcolor{improvegray}{$\pm 0.56$}}
& 55.42\\

DreamPRM
& 34.57{\small\textcolor{improvegray}{$\pm 0.63$}}
& 50.26{\small\textcolor{improvegray}{$\pm 1.12$}}
& 68.64{\small\textcolor{improvegray}{$\pm 1.06$}}
& 67.93{\small\textcolor{improvegray}{$\pm 0.65$}}
& 55.35\\

\rowcolor{lightpurple}
BayesianPRM
& \textbf{35.68}{\small\textcolor{improvegray}{$\pm 0.58$}}
& \textbf{50.67}{\small\textcolor{improvegray}{$\pm 0.60$}}
& 68.92{\small\textcolor{improvegray}{$\pm 0.51$}}
& \textbf{69.00}{\small\textcolor{improvegray}{$\pm 0.64$}}
& \textbf{56.07}\\
\bottomrule
\end{tabular}
}
\label{tab:tts}
\end{table*}

\subsection{Test-time Scaling Performance}
\label{sec:tts_performance}
\textbf{Experimental Setup.} We use the same $\mathcal{D}_\text{train}$ and $\mathcal{D}_\text{val}$ as in Section~\ref{sec:accuracy}. The test dataset $\mathcal{D}_\text{test}$ consists of the image-question pairs from the evaluation benchmarks. We compare the performance of Best-of-\(N\) and Instance-Adaptive Scaling (IAS)~\citep{park2026know}. For Best-of-\(N\), we set \(N=16\). For each sample in $\mathcal{D}_\text{test}$, InternVL3-8B generates 16 candidate responses, which are scored by the PRMs, and the response with the highest score is selected. For IAS, we set $N_\text{max}=16$ and dynamically allocate the rollout budget; the detailed procedure is provided in Appendix~\ref{sec:app_baselines}. We use Standard PRM, BetaPRM, and DreamPRM as baselines. For Best-of-$N$, we add an additional Self-Consistency~\citep{wang2022self} baseline.
\begin{wrapfigure}{r}{0.56\textwidth}
    \centering
    \vspace{-10pt}

    \begin{minipage}[t]{0.48\linewidth}
        \vspace{0pt}
        \centering

        \includegraphics[width=\linewidth]{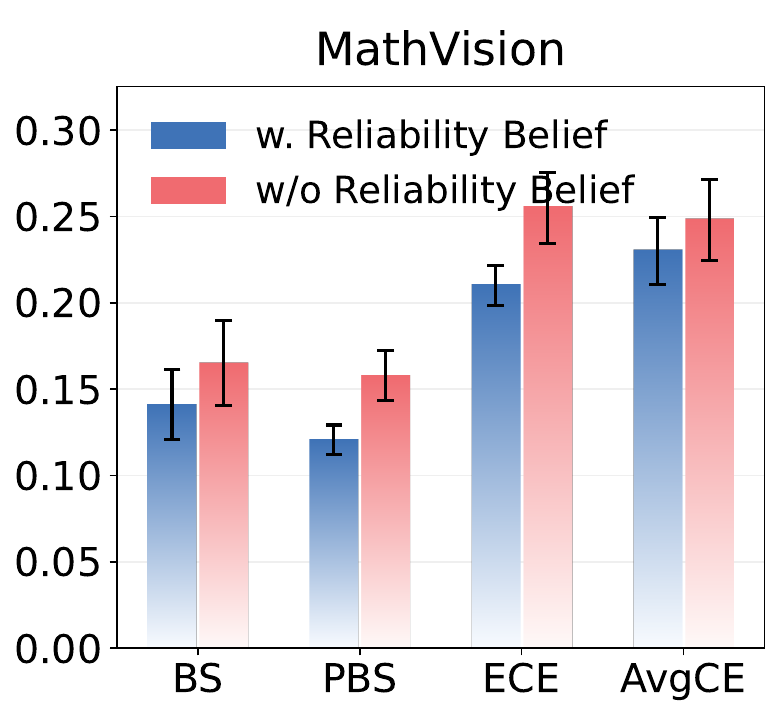}

        \captionof{figure}{PRM calibration results comparison on MathVision dataset.}
        \label{fig:cal_rel}
    \end{minipage}
    \hfill
    \begin{minipage}[t]{0.48\linewidth}
        \vspace{0pt}
        \centering

        \includegraphics[width=\linewidth]{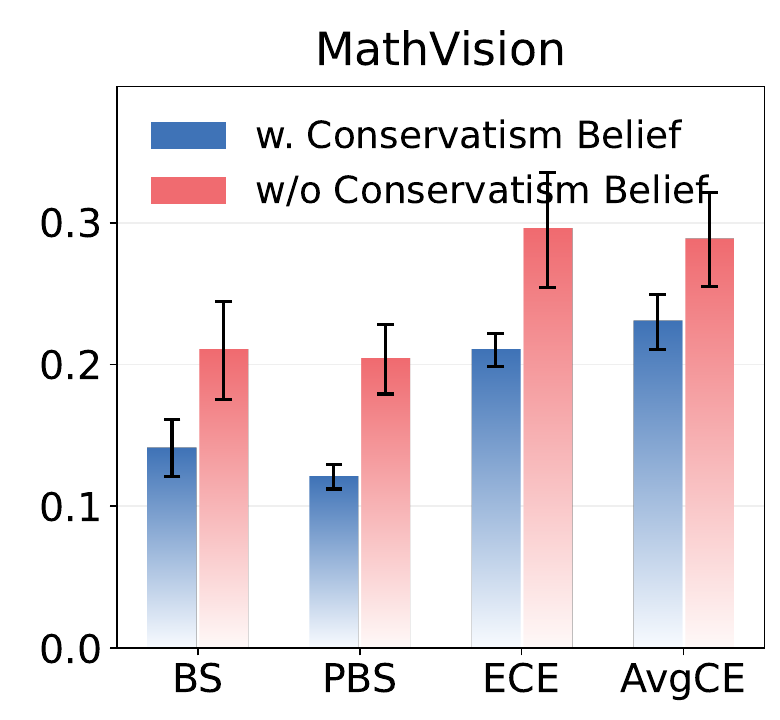}

        \captionof{figure}{PRM calibration results comparison on MathVision dataset.}
        \label{fig:cal_con}
    \end{minipage}

    \vspace{10pt}

    \begin{minipage}[t]{0.48\linewidth}
        \vspace{0pt}
        \centering
        \footnotesize
        \setlength{\tabcolsep}{3pt}
        \captionof{table}{Best-of-16 results with and without reliability belief.}
        \begin{tabular}{lcc}
            \toprule
            & \textbf{w. Rel.} & \textbf{w/o Rel.} \\
            \midrule
            Vision & \textbf{35.86} & 35.11 \\
            Verse  & \textbf{51.40} & 50.87 \\
            Vista  & \textbf{69.40} & 68.94 \\
            Star   & \textbf{69.33} & 68.56 \\
            \bottomrule
        \end{tabular}

        \label{tab:ablation_rel}
    \end{minipage}
    \hfill
    \begin{minipage}[t]{0.48\linewidth}
        \vspace{0pt}
        \centering
        \footnotesize
        \setlength{\tabcolsep}{3pt}
        \captionof{table}{Best-of-16 results with and without conservatism belief.}
        \begin{tabular}{lcc}
            \toprule
            & \textbf{w. Con.} & \textbf{w/o Con.} \\
            \midrule
            Vision & \textbf{35.86} & 35.09 \\
            Verse  & \textbf{51.40} & 50.83 \\
            Vista  & \textbf{69.40} & 68.80 \\
            Star   & \textbf{69.33} & 68.22 \\
            \bottomrule
        \end{tabular}

        \label{tab:ablation_con}
    \end{minipage}

    \vspace{-8pt}
\end{wrapfigure}


\textbf{Results.} As shown in Table~\ref{tab:tts}, BayesianPRM enhances test-time scaling performance under both Best-of-\(N\) and IAS. Under Best-of-$N$ with $N=16$, BayesianPRM achieves the highest average accuracy of \textbf{56.50}, outperforming Standard PRM, BetaPRM, and DreamPRM by 1.54, 0.85, and 0.92 points, respectively. BayesianPRM obtains the best performance on MathVision, MathVista, and MMStar, while remaining competitive on MathVerse-VO, where it is only 0.45 points below BetaPRM. With IAS and \(N_{\max}=16\), BayesianPRM again achieves the highest average accuracy of \textbf{56.07}, exceeding Standard PRM, BetaPRM, and DreamPRM by 1.14, 0.65, and 0.72 points, respectively. It achieves the best result on MathVision, MathVerse-VO, and MMStar. On MathVista, BayesianPRM reaches 68.92, only 0.19 points below Standard PRM. 
\subsection{Ablation Study}

In this section, we conduct an ablation study for BayesianPRM. We mainly examine the necessity of the reliability-aware belief updating and conservatism-aware belief calibration.

\textbf{Reliability-Aware Belief Updating.} To investigate the necessity of reliability-aware belief updating, we first train the reward ensemble using the standard cross-entropy loss, and then directly reweight the uniform prior toward lower reward predictions, omitting the belief network training step. Figure~\ref{fig:cal_rel} compares the calibration metrics on MathVision with and without reliability-aware belief updating, where BS, PBS, AvgCE denote Brier Score, Positive Brier Score, and Average Calibration Error. Removing reliability-aware belief updating increases all four metrics, indicating worse calibration and less accurate reward estimates. Table~\ref{tab:ablation_rel} compares the Best-of-16 results on four datasets, where Vision, Verse, Vista, and Star correspond to MathVision, MathVerse-VO, MathVista, and MMStar, respectively. Performance consistently drops after removing reliability-aware belief updating. Together, these results show that reliability-aware belief updating is important for both reward calibration and downstream test-time scaling.

\textbf{Conservatism-Aware Belief Calibration.} To examine the necessity of conservatism-aware belief calibration, we directly use the reliability-aware belief for reward estimation. Figure~\ref{fig:cal_con} and Table~\ref{tab:ablation_con} report the calibration results on MathVision and the test-time scaling results across four benchmarks, respectively. It is clear that removing conservatism-aware belief calibration consistently degrades both calibration and test-time scaling performance, indicating the importance of this step. More ablation study results are deferred to Appendix~\ref{sec:app_ablation} due to space limit.

\section{Conclusion}

In this work, we investigate the uncertainty in process reward modeling and identify two main sources: PRM uncertainty and policy uncertainty. To jointly address both uncertainties, we propose BayesianPRM, which formulates process reward modeling as Bayesian learning and consists of two key steps: reliability-aware belief updating and conservatism-aware belief calibration. Experimental results across various multimodal reasoning benchmarks indicate that BayesianPRM consistently improves both PRM reward prediction and test-time scaling performance.

\section{AI Use Statement}

In this work, we use LLMs to improve the grammar and presentation of the manuscript and to assist with parts of the experimental code. We affirm that LLMs were not involved in the core research contributions, including the conceptualization of the idea, the development of methodology, and theoretical proofs. All LLM-assisted content and code were reviewed and verified by the authors. The authors take full responsibility for the final content of this work.

\section{Ethics Statement}

This research does not involve human subjects, animal experiments, or sensitive datasets, nor does it
include applications that may pose ethical risks. We believe that this study does not raise any ethical
concerns.

\section{Reproducibility Statement}

We have made efforts to ensure that our work is reproducible. We provide the training scripts of BayesianPRM in the supplementary materials. Additionally, we provide the detailed experimental setup in Appendix~\ref{sec:app_setup}, and the implementation details in Appendix~\ref{sec:app_details}.

\bibliography{iclr2027_conference}
\bibliographystyle{iclr2027_conference}

\newpage
\appendix

\section{Related Work}

\textbf{Process Reward Models.} Reward models play a crucial role in reinforcement learning (RL)~\citep{schulman2017proximal,ahmadian2024back,hu2025reinforce++,guo2025deepseek} and test-time scaling~\citep{snell2024scaling,huang2025efficient,chae2026web,you2026parallel} for reasoning tasks. Process reward models (PRMs)~\citep{lightman2024let,du2025mm,sun2025freeprm,zheng2025survey,pala2025error} provide step-level rewards for intermediate reasoning steps, and have been shown more effective than outcome reward models (ORMs)~\citep{uesato2022solving,cobbe2021training,shao2024deepseekmath}. PRMs can be categorized into value-based PRMs~\citep{li2026process,li2026training,park2026know} and generative PRMs~\citep{zhao2026genprm,khalifa2025process,liang2025generative,zhang2026gm}. Our work falls into the scope of value-based PRMs, where PRMs are required to predict the process reward that reflects the success rate of a given prefix estimated by Monte-Carlo sampling~\citep{wang2024math,luo2024improve}. However, due to the inherent limitations of supervised learning~\citep{hastie2008overview}, value-based PRMs may produce inaccurate process rewards. To address this, recent works~\citep{li2026process,park2026know,ma2026distributional} have modeled PRMs as distributions rather than point estimates. BayesianPRM can also be seen as a distributional PRM, where the posterior belief represents the probability of different reward hypotheses, and thus induces a distribution over possible process reward values.

\textbf{Test-Time Scaling.} Test-time scaling~\citep{snell2024scaling,brown2024large,gou2024critic,zhang2025survey} improves reasoning quality by allocating additional computation at inference time. Representative methods include \textit{parallel scaling} such as self-consistency~\citep{brown2024large,irvine2023rewarding,song2025good,nguyen2024consistent,wu2026more}, PlanSearch~\citep{wang2025planning}, and \textit{sequential scaling} such as self-refine~\citep{zhang2024small,madaan2023self}, sequential revision~\citep{lee2025evolving}, ReAct~\citep{yao2022react}, and so on. A common and simple method is Best-of-$N$~\citep{lightman2024let}, which selects the best solution among multiple candidates evaluated by a verifier or reward model. Some works have also explored to adaptively allocate the compute budget to improve computational efficiency, such as Instance-Adaptive Scaling (IAS)~\citep{park2026know} and Adaptive Computation Allocation (ACA)~\citep{li2026process}.

\textbf{Uncertainty Quantification for LLMs.} Uncertainty quantification plays a crucial role in assessing the confidence of LLM outputs~\citep{shorinwa2025survey}. Prior studies~\citep{han2024towards,han2026uncertainty} have shown a strong correlation between the uncertainty and correctness of LLM responses. Major uncertainty quantification methods for LLMs include (1) \textit{white-box methods}, which quantify the uncertainty with visible metrics such as entropy~\citep{malinin2020uncertainty}, perplexity~\citep{fadeeva2024fact}, token log-likelihood~\citep{manakul2023selfcheckgpt}, and so on; (2) \textit{black-box methods}, which quantify the uncertainty entirely from the LLM responses to a given prompt. Representative methods include self-consistency~\citep{brown2024large}, self-verbalized methods~\citep{band2024linguistic,lin2022teaching,mielke2022reducing,guo2026llms} that train the LLM to provide a natural-language estimate of its
confidence, and so on. While these studies primarily focus on quantifying the uncertainty of LLM outputs, our work explores to capture the uncertainty of PRMs with Bayesian learning.

\textbf{Ensemble Reward Models.} Ensemble learning~\citep{dong2020survey,krogh1994neural} is a classical machine learning approach that improves predictive performance and uncertainty estimation for neural networks. Previous studies~\citep{coste2024reward,eisenstein2023helping} have shown that ensemble reward models could effectively mitigate reward over-optimization~\citep{gao2023scaling}. Among all prior works, HyRe~\citep{lee2026inference} and ActivePRM~\citep{duan2025efficient} are the closest to our work. HyRe adjusts the weights of reward ensemble at test time to adapt to different preferences, while ActivePRM leverages ensemble disagreement for active learning. In contrast, BayesianPRM captures uncertainty through Bayesian belief updating at both training and test time, distinguishing our approach from prior methods.

\section{Proofs}

\subsection{Proof of Proposition~\ref{prop:bayes_optimal}}

\begin{proof}
Since the objective in Equation~\ref{eq:elbo} is separable across different prefixes, we can optimize the reliability-aware belief for each prefix $c_t$ independently. For a fixed $c_t$, minimizing Equation~\ref{eq:elbo} is equivalent to solving
\begin{equation*}
\bm{\alpha}^\text{rel}(c_t)=\arg\max_{\bm{\alpha}\in\Delta^{M-1}}
\left\{
\sum_{m=1}^{M}\alpha_m\ell_{t,m}
-
\beta_1
D_{\mathrm{KL}}
\left(
\bm{\alpha}\Vert
\bm{\alpha}^{\text{prior}}(c_t)
\right)
\right\},
\end{equation*}
where $\Delta^{M-1}=\{\bm{\alpha}\in\mathbb{R}_{\geq 0}^{M}:\sum_{m=1}^{M}\alpha_m=1\}$ denotes the probability simplex. Expanding the KL divergence gives
\begin{equation*}
\bm{\alpha}^\text{rel}(c_t)=\arg \max_{\bm{\alpha}\in\Delta^{M-1}}
\left\{
\sum_{m=1}^{M}\alpha_m\ell_{t,m}
-
\beta_1
\sum_{m=1}^{M}
\alpha_m
\log
\frac{\alpha_m}
{\alpha_m^{\text{prior}}(c_t)}
\right\}.
\end{equation*}

Introducing a Lagrange multiplier $\lambda$ for the constraint
$\sum_{m=1}^{M}\alpha_m=1$ and multipliers
$\{\mu_m\}_{m=1}^{M}$ for the constraints
$\alpha_m\geq 0$, the Lagrangian function is
\begin{equation*}
\begin{aligned}
\mathcal{J}(\bm{\alpha},\lambda,\bm{\mu})
=
\sum_{m=1}^{M}\alpha_m\ell_{t,m}
-
\beta_1
\sum_{m=1}^{M}
\alpha_m
\log
\frac{\alpha_m}
{\alpha_m^{\text{prior}}(c_t)}
+
\lambda
\left(
\sum_{m=1}^{M}\alpha_m-1
\right)
+
\sum_{m=1}^{M}\mu_m\alpha_m,
\end{aligned}
\end{equation*}
where $\mu_m\geq 0$ for all $m$. Due to the presence of the KL term, the optimal solution lies in the interior of the
probability simplex, i.e.,
$\alpha_m>0$ for every $m$. By complementary slackness,
$\mu_m\alpha_m=0$, which immediately implies
$\mu_m=0$. Taking the derivative of the Lagrangian function and setting it to zero yields:
\begin{equation*}
\ell_{t,m}
-\beta_1
\left(
\log
\frac{\alpha_m}
{\alpha_m^{\text{prior}}(c_t)}
+1
\right)
+\lambda=0.
\end{equation*}
Therefore,
\begin{equation*}
\alpha_m
=
\alpha_m^{\text{prior}}(c_t)
\exp\left(\frac{\ell_{t,m}}{\beta_1}\right)
\exp\left(\frac{\lambda-\beta_1}{\beta_1}\right).
\end{equation*}
The last exponential term is shared by all reward hypotheses and can be determined by the normalization constraint. Enforcing
$\sum_{m=1}^{M}\alpha_m=1$ gives
\begin{equation*}
\alpha_m^{\text{rel}\star}(c_t)
= \alpha_m(c_t) =
\frac{
\alpha_m^{\text{prior}}(c_t)
\exp\left(\ell_{t,m}/\beta_1\right)
}{
\sum_{j=1}^{M}
\alpha_j^{\text{prior}}(c_t)
\exp\left(\ell_{t,j}/\beta_1\right)
}.
\end{equation*}
Since
\begin{equation*}
\ell_{t,m}
=
\log p(K_t\mid N_t,c_t,\theta_m),
\end{equation*}
we further obtain
\begin{equation*}
\alpha_m^{\text{rel}\star}(c_t)
\propto
\alpha_m^{\text{prior}}(c_t)
\cdot
p(K_t\mid N_t,c_t,\theta_m)^{1/\beta_1}.
\end{equation*}

For $\beta_1>0$, the objective is strictly concave with respect to
$\bm{\alpha}$ in the interior of the simplex, since the KL term is strictly convex. Hence, the stationary point above is the unique global optimum. In particular, when $\beta_1=1$, we recover
\begin{equation*}
\alpha_m^{\text{rel}\star}(c_t)
\propto
\alpha_m^{\text{prior}}(c_t)
p(K_t\mid N_t,c_t,\theta_m),
\end{equation*}
which is exactly the standard Bayesian posterior update.
    
\end{proof}

\subsection{Proof of Proposition~\ref{prop:likelihood}}
\label{app:proof_xi}
\begin{proof}
    For a prefix $c_t$, let 
$s_{(1)}\leq s_{(2)}\leq\cdots\leq s_{(M)}$
denote the order statistics of
$\{s(\theta_m;c_t)\}_{m=1}^{M}$,
where
$\theta_1,\ldots,\theta_M\overset{\mathrm{i.i.d.}}{\sim}p(\theta)$.
Let $\mathcal F(s)$ denote the cumulative distribution function of $s(\theta;c_t)$ induced by the prior $p(\theta)$, respectively.

According to the standard distribution of order statistics, the density of a hypothesis associated with the $m$-th ordered score can be obtained by reweighting the prior hypothesis density according to the probability that its score occupies the $m$-th rank:
\begin{equation}
\label{eq:pdf_order}
p_{(m)}(\theta)
=
\frac{M!}{(m-1)!(M-m)!}
[\mathcal F(s(\theta;c_t))]^{m-1}
[1-\mathcal F(s(\theta;c_t))]^{M-m}
p(\theta).
\end{equation}

Let $\theta_{(m)}$ and $\alpha^\text{hyb}_{(m)}(c_t)$ denote the reward hypothesis and hybrid belief associated with $s_{(m)}$, respectively. We define the weighting coefficients $\{\bar{\alpha}_m(\theta;c_t)\}_{m=1}^M$ as
\begin{equation*}
\bar{\alpha}_{m}(\theta;c_t)
=
\mathbb{E}_{\theta_{1:M}\sim p(\theta)}
\left[
\alpha_{(m)}^{\text{hyb}}(c_t)
\mid
\theta_{(m)}=\theta
\right],\qquad m=1,2,...,M.
\end{equation*}
Intuitively, $\bar{\alpha}_m(\theta;c_t)$ measures the expected belief weight received by a hypothesis $\theta$ when it occupies the $m$-th rank under the score function $s(\theta;c_t)$. Recall that our goal is to construct a posterior belief
$\tilde p(\theta|c_t)$ such that the expected reward under this posterior
is equivalent to the weighted expectation over the reward
hypotheses sampled from the prior $p(\theta)$. We first consider the latter:
\begin{equation}
\begin{aligned}
\mathbb{E}_{\theta_{1:M}\overset{\mathrm{i.i.d.}}{\sim}p(\theta)}
\left[
\sum_{m=1}^{M}
\alpha_{m}^{\text{hyb}}(c_t)
r_{m}(c_t)
\right]&=\mathbb{E}_{\theta_{1:M}\overset{\mathrm{i.i.d.}}{\sim}p(\theta)}
\left[
\sum_{m=1}^{M}
\alpha_{(m)}^{\text{hyb}}(c_t)
r_{\theta_{(m)}}(c_t)
\right]\\
&=\sum_{m=1}^M\mathbb{E}_{\theta_{1:M}\overset{\mathrm{i.i.d.}}{\sim}p(\theta)}\left[\alpha_{(m)}^{\text{hyb}}(c_t)
r_{\theta_{(m)}}(c_t)\right]\\
&=\sum_{m=1}^M\mathbb{E}_{\theta_{(m)}}\left[r_{\theta_{(m)}}\mathbb{E}_{\theta_{1:M}\overset{\mathrm{i.i.d.}}{\sim}p(\theta)}\left[\alpha^\text{hyb}_{(m)}(c_t)\mid\theta_{(m)}\right]\right]\\
&=\sum_{m=1}^M\mathbb{E}_{\theta_{(m)}}\left[\bar{\alpha}_m(\theta_{(m)};c_t)r_{\theta_{(m)}}\right].
\end{aligned}
\label{eq:weighted_order_expectation}
\end{equation}
The third equation holds by the law of total expectation. According to the probability density of the $m$-th order statistic in Equation~\ref{eq:pdf_order}, we have
\begin{equation}
\label{eq:former}
    \begin{aligned}
        &\sum_{m=1}^M\mathbb{E}_{\theta_{(m)}}\left[\bar{\alpha}_m(\theta_{(m)};c_t)r_{\theta_{(m)}}\right]=\sum_{m=1}^M\int_{\Theta}
r_{\theta}(c_t)
\bar{\alpha}_{m}(\theta;c_t)
p_{(m)}(\theta)
d\theta \\
&=\sum_{m=1}^{M}
\int_{\Theta}
r_{\theta}(c_t)
\bar{\alpha}_{m}(\theta;c_t)
\frac{M!}{(m-1)!(M-m)!}\cdot
[\mathcal F(s(\theta;c_t))]^{m-1}
[1-\mathcal F(s(\theta;c_t))]^{M-m}
p(\theta)d\theta \\
&=
\int_{\Theta}
r_{\theta}(c_t)
\sum_{m=1}^{M}\bar{\alpha}_{m}(\theta;c_t)
\frac{M!}{(m-1)!(M-m)!}\cdot
[\mathcal F(s(\theta;c_t))]^{m-1}
[1-\mathcal F(s(\theta;c_t))]^{M-m}
p(\theta)d\theta.
    \end{aligned}
\end{equation}

On the other hand, the posterior belief satisfies $\tilde p(\theta|c_t) = p(\theta)\xi(\theta;c_t)$, and therefore,
\begin{equation}
\label{eq:latter}
\mathbb{E}_{\theta\sim\tilde p(\theta|c_t)}
[r_\theta(c_t)]
=
\int_{\Theta}
r_\theta(c_t)
\xi(\theta;c_t)
p(\theta)d\theta .
\end{equation}

By comparing the two integral forms in Equation~\ref{eq:former} and Equation~\ref{eq:latter}, we obtain
\begin{equation*}
\xi(\theta;c_t)
=
\sum_{m=1}^{M}
\bar{\alpha}_{m}(\theta;c_t)
\frac{M!}{(m-1)!(M-m)!}
[\mathcal F(s(\theta;c_t))]^{m-1}
[1-\mathcal F(s(\theta;c_t))]^{M-m},
\end{equation*}

which is exactly the likelihood ratio stated in Proposition~\ref{prop:likelihood}.

The last step is to verify that $\tilde p(\theta|c_t)=p(\theta)\xi(\theta;c_t)$ is a valid probability distribution, that is, $\int_{\Theta}
p(\theta)\xi(\theta;c_t)d\theta=1$. The derivation is presented as follows.
\begin{equation*}
\begin{aligned}
\int_{\Theta}
p(\theta)\xi(\theta;c_t)d\theta
&=\sum_{m=1}^M\int_{\Theta}\bar{\alpha}_m(\theta;c_t)\cdot
\frac{M!}{(m-1)!(M-m)!}
[\mathcal F(s(\theta;c_t))]^{m-1}
[1-\mathcal F(s(\theta;c_t))]^{M-m}
p(\theta)d\theta\\
&=\sum_{m=1}^M\int_{\Theta}\bar{\alpha}_m(\theta;c_t)\cdot p_{(m)}(\theta) d\theta\\
&=\sum_{m=1}^M\int_{\Theta}\mathbb{E}_{\theta_{1:M}}\left[\alpha^\text{hyb}_{(m)}(c_t)\mid\theta_{(m)}=\theta\right]p_{(m)}(\theta)d\theta\\
&=\sum_{m=1}^M\mathbb{E}_{\theta_{(m)}}\left[\mathbb{E}_{\theta_{1:M}}\left[\alpha^\text{hyb}_{(m)}(c_t)\mid\theta_{(m)}=\theta\right]\right]\\
&=
\sum_{m=1}^{M}
\mathbb{E}_{\theta_{1:M}}
[
\alpha_{(m)}^{\text{hyb}}(c_t)
]\\
&=
\mathbb{E}_{\theta_{1:M}}
\left[
\sum_{m=1}^{M}
\alpha_{(m)}^{\text{hyb}}(c_t)
\right]\\
&=1.
\end{aligned}
\end{equation*}
The fifth equation holds by the law of total expectation, and the last equation holds since $\sum_{m=1}^M\alpha^\text{hyb}_{(m)}(c_t)=1$. Hence, $\tilde p(\theta|c_t)=p(\theta)\xi(\theta;c_t)$ defines a valid posterior belief, then we conclude the proof.

\end{proof}

\section{Experimental Setup}
\label{sec:app_setup}
\subsection{Backbones and Datasets}

\textbf{Backbones.} We instantiate BayesianPRM and other baselines including standard PRM, BetaPRM and DreamPRM with the same InternVL3-8B backbone across all experiments. For each method, we insert a \texttt{<prm>} marker after the reasoning prefix \(c_t\), and supervise the PRMs to predict the process reward at this marker position. For BayesianPRM, we attach multiple lightweight linear heads to the marker hidden state to instantiate the reward ensemble. Compared with the base model, these additional heads introduce only a negligible number of parameters. During fine-tuning, we freeze the vision encoder while updating the language model and the added linear layers. All experiments are conducted on 4$\times$H100 GPUs, and a complete fine-tuning run takes approximately 20 hours.

\textbf{Training Datasets.} Following BetaPRM, we train all PRMs on VisualPRM400k-v1.1, which is constructed using Monte-Carlo rollouts generated by the InternVL2.5 series models. We use the version ``-v1.1" because it provides the success-count information required for training BayesianPRM. For each reasoning prefix, the dataset reports not only the empirical success rate but also the number of successful trajectories among 16 Monte-Carlo rollouts, which is required to compute the binomial likelihood in Equation~\ref{eq:elbo}. VisualPRM400K-v1.1 comprises multiple sub-datasets across diagram understanding, chart and document
QA, general visual question answering, science reasoning, and mathematical/geometry reasoning. For each sub-dataset, we randomly split the data at the problem level, assigning 80\% to \(\mathcal{D}_{\text{train}}\) and the remaining 20\% to \(\mathcal{D}_{\text{val}}\). The reward ensemble is trained on \(\mathcal{D}_{\text{train}}\) using the standard cross-entropy loss, while the belief network is trained on the held-out \(\mathcal{D}_{\text{val}}\) according to Equation~\ref{eq:elbo}. We use different data for the two training stages to prevent the belief network from overfitting to samples already seen by the reward ensemble, which could lead to overly optimistic reliability estimates. Training the belief network on $\mathcal{D}_\text{val}$ allows it to learn reliability patterns that generalize to unseen reasoning prefixes. For the other baselines, we use the full dataset for training.

\textbf{Evaluation Datasets.} In Sections~\ref{sec:accuracy} and~\ref{sec:tts_performance}, we evaluate PRMs on MathVision, MathVerse, MathVista, and MMStar. We construct the evaluation datasets for each benchmark as follows. We first sample a subset of image-question pairs from the original datasets and use InternVL3-8B to generate a response for each pair. Each response is segmented into reasoning steps, from which we extract reasoning prefixes at different positions along the reasoning trajectory. To control the annotation cost, we randomly sample 1,000 prefixes from the resulting pool for each benchmark. For every selected prefix, we use InternVL3-8B to generate 16 Monte-Carlo rollouts. Each rollout is judged by Qwen2.5-32B-Instruct for correctness against the ground-truth answer. We then compute the empirical success rate of each prefix as the fraction of correct rollouts among the 16 samples. Finally, the reasoning prefixes and corresponding empirical success rates constitute the evaluation datasets.

\subsection{Evaluation Metrics}
\label{sec:app_metrics}
In this part, we introduce the calculation of the PRM accuracy evaluation metrics used in Section~\ref{sec:accuracy}. Let $\hat{q}_t\in[0,1]$ denote the process reward predicted by a PRM for the reasoning prefix $c_t$, and let $q_t\in[0,1]$ denote its empirical success
rate estimated from Monte-Carlo rollouts. We use the following metrics to
evaluate the accuracy of PRM predictions. Lower values indicate
better performance for all metrics. We use $B=10$ bins for all bin-based
calibration metrics.

\textbf{Brier Score.}
The Brier Score measures the mean squared error between the predicted process
reward and the empirical success rate:
\begin{equation*}
    \mathrm{Brier}
    =
    \frac{1}{T}
    \sum_{t=1}^{T}
    \left(\hat{q}_t-q_t\right)^2.
\end{equation*}
It measures the overall accuracy of probabilistic reward predictions, regardless
of whether the prediction error corresponds to overestimation or
underestimation.

\textbf{Positive Brier Score.}
To specifically quantify reward overestimation, we further use the Positive
Brier Score:
\begin{equation*}
    \mathrm{PosBrier}
    =
    \frac{1}{T}
    \sum_{t=1}^{T}
    \left[
    \max\left(\hat{q}_t-q_t,0\right)
    \right]^2.
\end{equation*}
Different from the standard Brier Score, Positive Brier only penalizes cases
where the predicted process reward exceeds the empirical success rate, and
therefore directly reflects the severity of reward overestimation.

\textbf{Expected Calibration Error (ECE).}
ECE measures the discrepancy between predicted rewards and empirical success
rates after partitioning the prediction interval $[0,1]$ into $B$ equal-width
bins. Let $\mathcal{I}_b$ denote the set of samples whose predictions fall into
the $b$-th bin. We define
\begin{equation*}
    \mathrm{Pred}(\mathcal{I}_b)
    =
    \frac{1}{|\mathcal{I}_b|}
    \sum_{t\in\mathcal{I}_b}
    \hat{q}_t,
    \qquad
    \mathrm{Succ}(\mathcal{I}_b)
    =
    \frac{1}{|\mathcal{I}_b|}
    \sum_{t\in\mathcal{I}_b}
    q_t.
\end{equation*}
ECE is then computed as
\begin{equation*}
    \mathrm{ECE}
    =
    \sum_{b=1}^{B}
    \frac{|\mathcal{I}_b|}{T}
    \left|
    \mathrm{Pred}(\mathcal{I}_b)
    -
    \mathrm{Succ}(\mathcal{I}_b)
    \right|.
\end{equation*}

\textbf{Average Calibration Error (AverageCE).}
AverageCE uses the same fixed-width bins as ECE, but assigns equal weight to
each non-empty bin instead of weighting bins by their sample counts. Let
$\mathcal{B}^{+}$ denote the set of non-empty bins. Then,
\begin{equation*}
    \mathrm{AverageCE}
    =
    \frac{1}{|\mathcal{B}^{+}|}
    \sum_{b\in\mathcal{B}^{+}}
    \left|
    \mathrm{Pred}(\mathcal{I}_b)
    -
    \mathrm{Succ}(\mathcal{I}_b)
    \right|.
\end{equation*}
Compared with ECE, AverageCE gives the same importance to calibration errors
in different prediction ranges regardless of how many samples fall into each
range.

\subsection{Baselines}

\label{sec:app_baselines}

\textbf{BetaPRM.} The core of BetaPRM~\citep{li2026process} lies in modeling the process reward as a Beta distribution, rather than using the point estimation. Specifically, BetaPRM assume a binomial generative process for the successful continuations: $K_t\mid q_t\sim\mathrm{Binomial}(N,q_t)$. It models $q_t$ with a Beta belief, $q_t\sim\mathrm{Beta}(\alpha_t,\beta_t)$. Then the Beta distribution is reparameterized by its mean $\mu_t=\alpha_t/(\alpha_t+\beta_t)$ and concentration $\kappa_t=\alpha_t+\beta_t$. At the $t$-th \texttt{<prm>} marker, the base model produces a hidden state $h_t$ and vocabulary logits $z_t$. $\mu_t$ is computed as:
\begin{equation*}
    \mu_t=\frac{\exp(z^{\mathrm{yes}}_t)}{\exp(z^{\mathrm{yes}}_t)+\exp(z^{\mathrm{no}}_t)},
\end{equation*}
where $z^{\mathrm{yes}}_t$ and $z^{\mathrm{no}}_t$ denote the logits of two tokens \texttt{yes} and \texttt{no}. The reliability is represented by the concentration parameter $\kappa_t$:
\begin{equation*}
    \kappa_t=\mathrm{softplus}(g_\phi(h_t))+\kappa_{\mathrm{min}},
\end{equation*}
where $g_\phi$ is a lightweight linear head and $\kappa_{\min}$ is a small positive number for numerical stability. Then $\alpha_t=\mu_t\kappa_t$ and $\beta_t=(1-\mu_t)\kappa_t$. 

BetaPRM is trained with the Beta-Binomial loss,
\begin{equation*}
    \mathcal{L}_{\text{Beta-Binomial}}=-\frac{1}{|P|}\sum_{t\in\mathcal{P}}\log p(K_t\mid N,\alpha_t,\beta_t),
\end{equation*}
where
\begin{equation*}
    p(K_t\mid N,\alpha_t,\beta_t)=\binom{N}{K_t}\frac{\text{Beta}(K_t+\alpha_t,N-K_t+\beta_t)}{\text{Beta}(\alpha_t,\beta_t)}
\end{equation*}
represents the predictive probability of the observed count. BetaPRM also adds a regularization loss to penalize the product of disagreement and concentration:
\begin{equation*}
    \mathcal{L}_\text{reg}=\lambda_\text{reg}\frac{1}{|\mathcal{P}|}\sum_{t\in\mathcal{P}}\left|\mathrm{sg}(\mu_t)-\frac{K_t}{N}\right|\kappa_t,
\end{equation*}
where $\mathrm{sg}(\cdot)$ is the stop-gradient operation. The overall training loss is
\begin{equation*}
    \mathcal{L}=\mathcal{L}_\text{Beta-Binomial}+\mathcal{L}_\text{reg}.
\end{equation*}
We reproduce BetaPRM following its official implementation\footnote{\textcolor{brown}{https://github.com/JinYuanLi0012/Beta-Binomial-PRM.git}}.

\textbf{DreamPRM.} DreamPRM~\citep{cao2026dreamprm} is a domain-reweighted multimodal process reward training framework that addresses the dataset quality imbalance issue in multimodal reasoning. DreamPRM dynamically adjusts the importance of different training domains with bi-level optimization~\citep{talbi2013taxonomy}. It solves the following optimization problem:
\begin{equation*}
    \begin{aligned}
&\min_{\alpha}\quad  \mathcal{L}_\text{meta}\bigl(\mathcal{D}_\text{meta},\phi^*(\alpha)\bigr) \\
\text{s.t.}\quad & \phi^*(\alpha)
= \arg\min_{\phi}\mathcal{L}_\text{tr}(\mathcal{D}_\text{tr},\phi,\alpha)
\end{aligned}
\end{equation*}
where $\mathcal{L}_\text{tr}(\mathcal{D}_\text{tr},\phi,\alpha)=\sum_{k=1}^K\alpha_k\mathcal{L}_\text{tr}(\mathcal{D}_k,\phi)$, $\alpha_k$ is the weight coefficient for dataset $\mathcal{D}_k$, $\mathcal{L}_\text{tr}(\mathcal{D}_k,\phi)$ is the standard PRM loss with MSE loss on $\mathcal{D}_k$. $\mathcal{D}_\text{meta}$ is a validation dataset for testing the generalization of PRMs. $\mathcal{L}_\text{meta}$ is the meta loss on $\mathcal{D}_\text{meta}$ to penalize the inconsistency of the aggregated PRM score with the ground-truth reward. We implement DreamPRM based on its official codebase\footnote{\textcolor{brown}{https://github.com/coder-qicao/DreamPRM.git}}, and replace its training dataset with VisualPRM400K-v1.1.

\textbf{QR-PRM.}
QR-PRM~\citep{park2026know} calibrates the PRM for a target policy
through quantile regression on a calibration set and further uses the calibrated reward for
instance-adaptive test-time scaling (IAS). Its key observation is that the PRMs often
overestimate the process reward due to mismatch of the rollout policy at test time and the policy for collecting the supervision data.

\textit{Calibration Set Construction.}
Given a target policy $\pi$, QR-PRM first constructs a policy-specific
calibration set. For each validation question $q$, $\pi$ generates
multiple reasoning trajectories. Each trajectory is decomposed into reasoning
prefixes, and for every prefix $c_t$, the empirical successful rate $\tilde{p}_t$ is computed with Monte-Carlo rollouts generated by $\pi$. The resulting tuples $\mathcal{D}_{\mathrm{cal}}
    =
    \left\{
    (c_t,\tilde{p}_t)
    \right\}$
constitute the calibration set used to adapt the PRM to the target policy.

\textit{Quantile Regression.} QR-PRM predicts quantiles of the process reward on $\mathcal{D}_\text{cal}$. Specifically, the
original PRM linear head is replaced with a quantile
prediction head. For a target quantile level $\tau_j$, the model predicts $\hat{r}_t^{(\tau_j)}
    =
    \sigma\left(g_{\phi,j}(h_t)\right)$,
where $h_t$ denotes the hidden representation of $c_t$ and
$g_{\phi}$ is the quantile prediction head. The PRM is fine-tuned on $\mathcal{D}_{\mathrm{cal}}$ using the weighted
quantile loss
\begin{equation*}
\mathcal{L}
=
\frac{1}{|\mathcal{D}_{\mathrm{cal}}|}
\sum_{t}
\frac{1}{J}
\sum_{j=1}^{J}
\left[
\tau_j
\max\left(0,\tilde{p}_t-\hat{r}_t^{(\tau_j)}\right)
+
(1-\tau_j)
\max\left(0,\hat{r}_t^{(\tau_j)}-\tilde{p}_t\right)
\right].
\end{equation*}
The asymmetric loss allows different quantiles of the success probability to
be estimated. The original work predicts the $10\%$, $50\%$, and $90\%$
quantiles. Since QR-PRM is not designed as a multimodal PRM, we implement our own multimodal version following its official codebase\footnote{\textcolor{brown}{https://github.com/young-j-park/prm-calibration.git}}.

\textit{Instance-Adaptive Scaling (IAS).}
QR-PRM further uses the calibrated process reward to dynamically allocate
test-time computation. For Best-of-$N$, given a calibrated reward
$\hat{r}^{(\tau)}$ and a target success probability $C$, the required number
of samples is estimated as
\begin{equation*}
    N_{\mathrm{IAS}}
    =
    \min
    \left(
    \left\lceil
    \frac{\log(1-C)}
         {\log(1-\hat{r}^{(\tau)})}
    \right\rceil,
    N_{\max}
    \right),
\end{equation*}
where $N_{\max}$ denotes the maximum sampling budget.
Therefore, more difficult questions receive
larger budgets. The original work adopts a conservative lower quantile
($\tau=0.1$) for IAS.

\textbf{Self-Consistency.} Self-Consistency~\citep{wang2022self} is a verifier-free test-time scaling baseline that improves reasoning by aggregating multiple sampled responses. Given a problem, we sample \(N\) candidate reasoning trajectories from the rollout policy and extract their final answers. The answer receiving the majority vote is selected as the final prediction. Unlike PRM-guided Best-of-\(N\), Self-Consistency does not rely on an external reward model and selects the output solely based on agreement among sampled responses.

\textbf{Temperature Scaling.}
Temperature scaling (TS) is a parametric post-hoc calibration method that
rescales the output logits of a pretrained PRM using a single temperature
parameter. Suppose the PRM produces two logits $z_t^{\mathrm{yes}}$ and
$z_t^{\mathrm{no}}$ at the \texttt{<prm>} marker. The calibrated process reward
is computed as
\begin{equation*}
    \hat r_t^{\mathrm{TS}}
    =
    \frac{
    \exp\left(z_t^{\mathrm{yes}}/T\right)
    }{
    \exp\left(z_t^{\mathrm{yes}}/T\right)
    +
    \exp\left(z_t^{\mathrm{no}}/T\right)
    },
\end{equation*}
where $T>0$ is the temperature parameter. The parameter $T$ is selected on the
calibration set by minimizing a calibration objective, such as the Brier score.
Temperature scaling preserves the ranking induced by the original logits while
adjusting the sharpness of the predicted probabilities.

\textbf{Isotonic Regression.}
Isotonic regression (IR) is a non-parametric calibration method that learns a
monotonic mapping from the original PRM prediction to the empirical success
probability. Given calibration pairs
$\{(\hat r_t,\tilde p_t)\}$, where $\hat r_t$ is the original PRM reward and
$\tilde p_t$ is the corresponding empirical success rate, isotonic regression
learns a non-decreasing function $g_{\mathrm{IR}}:[0,1]\rightarrow[0,1]$ by
solving
\begin{equation*}
    g_{\mathrm{IR}}
    =
    \arg\min_{g\in\mathcal{G}_{\mathrm{mono}}}
    \sum_{t}
    \left(
    g(\hat r_t)-\tilde p_t
    \right)^2,
\end{equation*}
where $\mathcal{G}_{\mathrm{mono}}$ denotes the set of non-decreasing
functions. The calibrated reward is then
\begin{equation*}
    \hat r_t^{\mathrm{IR}}
    =
    g_{\mathrm{IR}}(\hat r_t).
\end{equation*}
In practice, the mapping is fitted using the Pool Adjacent Violators Algorithm
(PAVA), resulting in a piecewise-constant monotonic calibration function.
Unlike temperature scaling, isotonic regression does not assume a specific
parametric form for the calibration mapping.

\textbf{Histogram Binning.}
Histogram binning (HB) is another non-parametric post-hoc calibration method.
It partitions the prediction interval $[0,1]$ into $B$ equal-width bins
$\{\mathcal{B}_b\}_{b=1}^{B}$. For each bin, the calibrated value is defined
as the average empirical success rate of the calibration samples whose original
PRM predictions fall into that bin:
\begin{equation*}
    \mu_b
    =
    \frac{1}{|\mathcal{I}_b|}
    \sum_{t\in\mathcal{I}_b}
    \tilde p_t,
    \qquad
    \mathcal{I}_b
    =
    \left\{
    t:\hat r_t\in\mathcal{B}_b
    \right\}.
\end{equation*}
For a test prefix whose original PRM prediction $\hat r_t$ falls into
$\mathcal{B}_b$, the calibrated process reward is
\begin{equation*}
    \hat r_t^{\mathrm{HB}}
    =
    \mu_b.
\end{equation*}
Thus, histogram binning replaces each raw prediction with the empirical success
rate observed in the corresponding confidence interval.

\section{Implementation Details}
\label{sec:app_details}
In this section, we present the implementation details of BayesianPRM. 

\textbf{Codebase.} BayesianPRM is implemented based on BetaPRM codebase\footnote{\textcolor{brown}{https://github.com/JinYuanLi0012/Beta-Binomial-PRM.git}}. We thank the authors of BetaPRM for releasing such a comprehensive training and evaluation pipeline.

\textbf{Model Architecture.} \textit{(1) Reward Ensemble}. The reward ensemble is built upon the InternVL3-8B backbone and an ensemble of \(M\) reward heads $\{g_m\}_{m=1}^M$. For each reasoning prefix \(c_t\), we extract the final-layer hidden representation \(z_t\in\mathbb{R}^{H_1}\) at the corresponding \texttt{<prm>} position. Each reward head is parameterized by a two-layer MLP with one hidden layer with dimension $H_2$. Instead of instantiating \(M\) separate modules, we stack the parameters of all ensemble members along an additional ensemble dimension, resulting in an ensemble architecture of $[M,H_1,H_2,1]$. We then apply a sigmoid activation to constrain the reward values to the range $[0,1]$. \textit{(2) Randomized Prior Network}. To improve the diversity of the reward ensemble, we associate the reward heads with an randomized prior network $\{f_m\}_{m=1}^M$. The prior network adopts the same MLP architecture of $[M,H_1,H_2,1]$ as the reward heads, while its parameters are randomly initialized and kept fixed. For the $m$-th ensemble member, the reward is given by $r_m(c_t)=\sigma(g_m(z_t)+\lambda_\text{prior}f_m(z_t))$. That is, the reward logit is jointly determined by the learned reward head and the corresponding prior network. Since different reward heads are equipped with distinct fixed prior functions, each member learns a residual correction on top of a different functional initialization. This discourages the ensemble members from collapsing to identical reward functions under the same supervision and thereby improves the diversity of the reward ensemble. \textit{(3) Belief Network}. The belief network is designed to estimate a prefix-dependent reliability distribution over the \(M\) reward hypotheses. For each reasoning prefix \(c_t\), we concatenate its hidden representation \(z_t\in\mathbb{R}^{H_1}\) with the reward predictions of all ensemble members, \(\bm{r}(c_t)=[r_1(c_t),\ldots,r_M(c_t)]^\top\), as the input to the belief network: $x_t=[z_t;\bm{r}(c_t)]\in \mathbb{R}^{H_1+M}$. The belief network $w_\nu(\cdot)$ is parameterized by a two-layer MLP with one hidden layer of dimension \(H_2\), resulting in an architecture of \([H_1+M,H_2,M]\). Specifically, the input $x_t$ is forwarded through \(w_v\), which outputs \(M\) belief logits. We then apply a softmax function over the ensemble dimension to obtain the reliability-aware belief $\bm{\alpha}^\text{rel}(c_t)=\mathrm{Softmax}(w_\nu(x_t))$. During belief-network training, the backbone, reward heads, and prior networks are kept frozen, and only the parameters of the belief network are optimized.

\textbf{Training Procedure.} The training procedure of BayesianPRM consists of two stages: reward ensemble training and belief network training. \textit{(1) Reward Ensemble Training.} In this stage, we freeze the vision encoder while fine-tuning the language model and all reward heads with $\mathcal{D}_\text{train}$. Each reward head is optimized using the standard cross-entropy loss:
\begin{equation}
\label{eq:full_bce}
        \mathcal{L}_\mathrm{CE}=-\frac{1}{M|\mathcal{D}_\text{train}|}\sum_{c_t\in\mathcal{D}_\text{train}}\sum_{m=1}^M\left[\hat{q}_t\log r_{m}(c_t)+(1-\hat{q}_t)\log(1-r_m(c_t))\right],
\end{equation}
where $\hat{q}_t$ is the Monte-Carlo label, $r_m(c_t)$ is the prediction of $m$-th reward head. \textit{(2) Belief Network Training}. During belief network training, we freeze the entire backbone and all reward heads, and update only the belief network $w_\nu$ on \(\mathcal{D}_{\text{val}}\) using the ELBO objective in Equation~\ref{eq:elbo}. After training, given a reasoning prefix \(c_t\), we obtain the original ensemble reward predictions $\bm{r}(c_t)$ together with the reliability-aware belief \(\boldsymbol{\alpha}^{\text{rel}}(c_t)\).

\textbf{Inference Procedure.} During inference, we first calibrate the reliability-aware belief with the conservatism-aware adjustment to obtain \(\boldsymbol{\alpha}^{\text{hyb}}(c_t)\), as defined in Equation~\ref{eq:hybrid_belief}. We then compute the final hybrid reward by taking a weighted sum of the reward ensemble predictions using \(\boldsymbol{\alpha}^{\text{hyb}}(c_t)\): $\tilde{r}(c_t)=[\bm{\alpha}^\text{hyb}(c_t)]^\top\bm{r}(c_t)$. For test-time scaling, we need to assign a single reward to each complete reasoning trajectory. Given a trajectory \(\tau\) consisting of \(T\) reasoning steps, we aggregate the process rewards by averaging them across all steps: $R(\tau)=\frac{1}{T}\sum_{t=1}^T \tilde{r}(c_t)$.

We present the pseudo-code of BayesianPRM in Algorithm~\ref{alg:pseudo_bayes}.

\begin{algorithm}[htbp]
\caption{BayesianPRM}
\label{alg:pseudo_bayes}
\begin{algorithmic}[1]
\STATE \textbf{Require:} Training dataset $\mathcal{D}_{\text{train}}=\{c_t,K_t,N_t\}_{t=1}^T$, validation dataset $\mathcal{D}_{\text{val}}=\{c_t,K_t,N_t\}_{t=1}^{T^\prime}$, LLM backbone
\STATE \textbf{Initialization:} ensemble size $M$, reward ensemble $\{r_{\theta_m}\}_{m=1}^M$, randomized prior network $\{f_{\phi_m}\}_{m=1}^M$, belief network $w_\nu$, reliability coefficient $\beta_1$, conservatism coefficient $\beta_2$, prior weighting coefficient $\lambda_\text{prior}$ 
\STATE \textbf{\textcolor{blue}{// Train the Reward Ensemble}}
\FOR{each gradient step}
\STATE Compute loss $        \mathcal{L}_\mathrm{CE}=-\mathbb{E}_{c_t\in\mathcal{D}_\text{train}}\frac{1}{M}\sum_{m=1}^M\left[\hat{q}_t\log r_{m}(c_t)+(1-\hat{q}_t)\log(1-r_m(c_t))\right]$
\STATE Update $\{r_{\theta_m}\}_{m=1}^M$ using $\mathcal{L}_\mathrm{CE}$
\ENDFOR
\STATE \textbf{\textcolor{blue}{// Train the Belief Network}}
\FOR{each gradient step}
    \STATE Compute loss $        \mathcal{L}_\mathrm{ELBO}=-\mathbb{E}_{c_t\in\mathcal{D}_\text{val}}\sum_{m=1}^M\alpha^\text{rel}_m(c_t)(K_t\log r_m(c_t)+(N_t-K_t)\log(1-r_m(c_t))$ $-\beta_1\log(M\alpha_m^\text{rel}(c_t)))$
    \STATE Update $w_\nu$ using $\mathcal{L}_\mathrm{ELBO}$
\ENDFOR
\STATE Obtain the reliability-aware belief $\bm{\alpha}^\text{rel}(c_t)$
\STATE \textbf{\textcolor{blue}{// Inference}}
\STATE Conservatively calibrate the reliability-aware belief, and obtain the hybrid belief ${\alpha_m^{\text{hyb}}(c_t)}
=
\frac{
\alpha_m^{\text{rel}}(c_t)
\exp\left(
-
r_m(c_t)/
\beta_2
\right)
}{
\sum_{j=1}^{M}
\alpha_j^{\text{rel}}(c_t)
\exp\left(
-
r_j(c_t)/
\beta_2
\right)
},\qquad m=1,2,...,M$
\STATE Obtain the hybrid reward by taking a weighted sum of the reward ensemble predictions with $\bm{\alpha}^\text{hyb}(c_t)$: $\tilde{r}(c_t)=[\bm{\alpha}^\text{hyb}(c_t)]^\top\bm{r}(c_t)$ 
\end{algorithmic}
\end{algorithm}

\textbf{Hyperparameter Setup.} We list the main hyperparameters of BayesianPRM in Table~\ref{tab:hyperparameters}.

\begin{table}[t]
\centering
\caption{Hyperparameters for BayesianPRM.}
\label{tab:hyperparameters}
\begin{tabular}{ll}
\toprule
\rowcolor{gray!10}
\textbf{Item} & \textbf{Value} \\
\midrule
Backbone & InternVL3-8B \\
Optimizer & AdamW~\citep{loshchilov2017decoupled} \\
Learning Rate & $1 \times 10^{-5}$ \\
Weight Decay & 0.05 \\
Learning Rate Scheduler & Cosine \\
Warmup Ratio & 0.05 \\
Epochs & 1 \\
Global Batch Size & 512 \\
Ensemble / Belief Data Split & 80\% / 20\% \\
Max Sequence Length & 8192 \\
Ensemble Size $M$ & $20$ \\
Backbone Hidden Dimension $H_1$ & 3584 \\
Reward Head / Belief Network Hidden Dimension $H_2$ & 256 \\
\midrule
Reliability Coefficient $\beta_1$ & $0.05$ \\
Conservatism Coefficient $\beta_2$ & $0.05$ \\
Prior Weighting Coefficient $\lambda_{\text{prior}}$ & $10$ \\
\bottomrule
\end{tabular}
\end{table}

\section{Additional Experimental Results}

In this section, we provide additional experimental results and analyses omitted from the main text due to space limits, offering a more comprehensive evaluation of BayesianPRM.

\subsection{Results with Qwen3VL-8B}

In Section~\ref{sec:accuracy} and Section~\ref{sec:tts_performance}, we use InternVL3-8B as the test-time rollout policy to evaluate PRM accuracy and test-time scaling performance. In this section, we replace it with Qwen3-VL-8B-Instruct~\citep{qwen3technicalreport} to examine the generalizability of BayesianPRM across different rollout backbones.

\subsubsection{Process Reward Accuracy Evaluation}
\label{sec:app_accuracy}

\begin{table*}[t]
\centering
\caption{
\textbf{PRM accuracy evaluation results for Qwen3VL-8B.} We evaluate Brier Score, Positive Brier Score, ECE, and AverageCE on four benchmarks: MathVision, the Vision Only subset of MathVerse (denoted as MathVerse-VO), MathVista, and MMStar. {\color{improvegray}$\downarrow$} indicates that lower values are better. $\pm$ denotes the standard deviation across three random training seeds.
}
\resizebox{\textwidth}{!}{
\begin{tabular}{lccccc}
\toprule

& \textbf{MathVision}
& \textbf{MathVerse-VO}
& \textbf{MathVista}
& \textbf{MMStar}
& \textbf{Average} \\
\midrule

\multicolumn{6}{c}{\textit{\textbf{Brier Score}}} \\
\midrule

Standard PRM
& 0.2540{\small\textcolor{improvegray}{$\pm 0.0128$}}
& 0.2695{\small\textcolor{improvegray}{$\pm 0.0210$}}
& 0.2223{\small\textcolor{improvegray}{$\pm 0.0143$}}
& 0.2506{\small\textcolor{improvegray}{$\pm 0.0088$}}
& 0.2491\\

BetaPRM
& 0.2854{\small\textcolor{improvegray}{$\pm 0.0259$}}
& 0.2901{\small\textcolor{improvegray}{$\pm 0.0220$}}
& 0.2267{\small\textcolor{improvegray}{$\pm 0.0166$}}
& 0.2540{\small\textcolor{improvegray}{$\pm 0.0192$}}
& 0.2641 \\

DreamPRM
& 0.2869{\small\textcolor{improvegray}{$\pm 0.0234$}}
& 0.3221{\small\textcolor{improvegray}{$\pm 0.0202$}}
& 0.2363{\small\textcolor{improvegray}{$\pm 0.0051$}}
& 0.2588{\small\textcolor{improvegray}{$\pm 0.0120$}}
& 0.2760\\

\rowcolor{lightpurple}
BayesianPRM
& \textbf{0.2002}{\small\textcolor{improvegray}{$\pm 0.0187$}}
& \textbf{0.2439}{\small\textcolor{improvegray}{$\pm 0.0261$}}
& \textbf{0.2118}{\small\textcolor{improvegray}{$\pm 0.0242$}}
& \textbf{0.2262}{\small\textcolor{improvegray}{$\pm 0.0105$}}
& \textbf{0.2205}\\

\midrule
\multicolumn{6}{c}{\textit{\textbf{Positive Brier Score}}} \\
\midrule

Standard PRM
& 0.2163{\small\textcolor{improvegray}{$\pm 0.0152$}}
& 0.2587{\small\textcolor{improvegray}{$\pm 0.0118$}}
& 0.1743{\small\textcolor{improvegray}{$\pm 0.0244$}}
& 0.2393{\small\textcolor{improvegray}{$\pm 0.0130$}}
& 0.2222 \\

BetaPRM
& 0.2569{\small\textcolor{improvegray}{$\pm 0.0226$}}
& 0.2413{\small\textcolor{improvegray}{$\pm 0.0174$}}
& 0.1890{\small\textcolor{improvegray}{$\pm 0.0125$}}
& 0.2448{\small\textcolor{improvegray}{$\pm 0.0096$}}
& 0.2330 \\

DreamPRM
& 0.2607{\small\textcolor{improvegray}{$\pm 0.0287$}}
& 0.3054{\small\textcolor{improvegray}{$\pm 0.0190$}}
& 0.1966{\small\textcolor{improvegray}{$\pm 0.0118$}}
& 0.2307{\small\textcolor{improvegray}{$\pm 0.0162$}}
& 0.2484\\

\rowcolor{lightpurple}
BayesianPRM
& \textbf{0.1393}{\small\textcolor{improvegray}{$\pm 0.0130$}}
& \textbf{0.1747}{\small\textcolor{improvegray}{$\pm 0.0231$}}
& \textbf{0.1410}{\small\textcolor{improvegray}{$\pm 0.0276$}}
& \textbf{0.2049}{\small\textcolor{improvegray}{$\pm 0.0133$}}
& \textbf{0.1650}\\

\midrule
\multicolumn{6}{c}{\textit{\textbf{ECE}}} \\
\midrule

Standard PRM
& 0.2809{\small\textcolor{improvegray}{$\pm 0.0236$}}
& 0.2231{\small\textcolor{improvegray}{$\pm 0.0204$}}
& 0.1628{\small\textcolor{improvegray}{$\pm 0.0181$}}
& 0.2364{\small\textcolor{improvegray}{$\pm 0.0042$}}
& 0.2258\\

BetaPRM
& 0.3371{\small\textcolor{improvegray}{$\pm 0.0177$}}
& 0.2590{\small\textcolor{improvegray}{$\pm 0.0262$}}
& 0.1689{\small\textcolor{improvegray}{$\pm 0.0213$}}
& 0.2368{\small\textcolor{improvegray}{$\pm 0.0160$}}
& 0.2505 \\

DreamPRM
& 0.3257{\small\textcolor{improvegray}{$\pm 0.0216$}}
& 0.2588{\small\textcolor{improvegray}{$\pm 0.0149$}}
& 0.1842{\small\textcolor{improvegray}{$\pm 0.0106$}}
& 0.2219{\small\textcolor{improvegray}{$\pm 0.0165$}}
& 0.2477\\

\rowcolor{lightpurple}
BayesianPRM
& \textbf{0.1763}{\small\textcolor{improvegray}{$\pm 0.0115$}}
& \textbf{0.1485}{\small\textcolor{improvegray}{$\pm 0.0231$}}
& \textbf{0.1181}{\small\textcolor{improvegray}{$\pm 0.0144$}}
& \textbf{0.1898}{\small\textcolor{improvegray}{$\pm 0.0139$}}
& \textbf{0.1582}\\
\midrule
\multicolumn{6}{c}{\textit{\textbf{AverageCE}}} \\
\midrule

Standard PRM
& 0.2746{\small\textcolor{improvegray}{$\pm 0.0227$}}
& 0.1857{\small\textcolor{improvegray}{$\pm 0.0128$}}
& 0.1421{\small\textcolor{improvegray}{$\pm 0.0147$}}
& 0.2076{\small\textcolor{improvegray}{$\pm 0.0122$}}
& 0.2025\\

BetaPRM
& 0.2676{\small\textcolor{improvegray}{$\pm 0.0164$}}
& 0.2242{\small\textcolor{improvegray}{$\pm 0.0142$}}
& 0.1715{\small\textcolor{improvegray}{$\pm 0.0233$}}
& 0.1600{\small\textcolor{improvegray}{$\pm 0.0169$}}
& 0.2058 \\

DreamPRM
& 0.2856{\small\textcolor{improvegray}{$\pm 0.0221$}}
& 0.2113{\small\textcolor{improvegray}{$\pm 0.0157$}}
& 0.1742{\small\textcolor{improvegray}{$\pm 0.0217$}}
& 0.1820{\small\textcolor{improvegray}{$\pm 0.0070$}}
& 0.2133\\

\rowcolor{lightpurple}
BayesianPRM
& \textbf{0.2047}{\small\textcolor{improvegray}{$\pm 0.0121$}}
& \textbf{0.1405}{\small\textcolor{improvegray}{$\pm 0.0229$}}
& \textbf{0.1064}{\small\textcolor{improvegray}{$\pm 0.0063$}}
& \textbf{0.1581}{\small\textcolor{improvegray}{$\pm 0.0147$}}
& \textbf{0.1524}\\
\bottomrule
\end{tabular}
}
\label{tab:app_calibration}
\end{table*}

\textbf{Datasets.} The training dataset is kept consistent with Section~\ref{sec:accuracy}, using VisualPRM400K-v1.1. For BayesianPRM, we randomly use 80\% of the data to train the reward ensemble and the remaining 20\% to train the belief network. For the evaluation dataset, we still use MathVision, MathVerse-VO, MathVista and MMStar. For each image-question pair, we first use Qwen3VL-8B-Instruct to generate a reasoning response, from which intermediate reasoning prefixes
are extracted. Starting from each prefix $c_t$, we perform 16 Monte-Carlo rollouts using Qwen3VL-8B-Instruct. Each rollout is then evaluated by Qwen2.5-32B-Instruct as the judge model to determine its
correctness. The proportion of correct rollouts is treated as the empirical success rate $\hat{p}_t$. The
$(c_t, \hat{p}_t)$ pairs constitute $\mathcal{D}_\text{test}$ for evaluation.

\textbf{Evaluation Metrics.} Following Section~\ref{sec:accuracy}, we use Brier Score, Positive Brier Score, Expected Calibration Error (ECE) and Average Calibration Error (AverageCE) as the evaluation metrics.

\textbf{Baselines.} Consistent with Section~\ref{sec:accuracy}, we choose Standard PRM, BetaPRM, and DreamPRM as baselines.

\textbf{Results.} As shown in Table~\ref{tab:app_calibration}, BayesianPRM consistently achieves the best PRM accuracy across all four benchmarks and all evaluation metrics when Qwen3-VL-8B-Instruct is used as the rollout policy. Compared with Standard PRM, BayesianPRM reduces the average Brier Score from 0.2491 to 0.2205 and the Positive Brier Score from 0.2222 to 0.1650. More substantial improvements are observed in calibration metrics, where the average ECE decreases from 0.2258 to 0.1582 and AverageCE from 0.2025 to 0.1524. BayesianPRM also consistently outperforms BetaPRM and DreamPRM across all metrics. These results further demonstrate that the improved reward estimation and calibration of BayesianPRM generalize to different rollout backbones.

\subsubsection{Test-Time Scaling Performance}

\begin{table*}[t]
\centering
\caption{
\textbf{PRM-guided test-time scaling performance comparison.} We compare Best-of-\(N\) (\(N=16\)) and Instance-Adaptive Scaling (IAS, \(N_{\max}=16\)) on MathVision, MathVerse-VO, MathVista and MMStar.  {\color{improvegray}$\uparrow$} indicates that higher values are better. $\pm$ denotes the standard deviation across three random training seeds.
}
\resizebox{\textwidth}{!}{
\begin{tabular}{lccccc}
\toprule

& \textbf{MathVision}
& \textbf{MathVerse-VO}
& \textbf{MathVista}
& \textbf{MMStar}
& \textbf{Average} \\
\midrule

Qwen3VL-8B
& 20.07
& 50.51
& 70.30
& 56.67
& 49.39 \\

\midrule
\multicolumn{6}{c}{\textit{\textbf{Best-of-$N$ ($N=16$)}}} \\
\midrule

Self-Consistency
& \textbf{22.04}
& 54.19
& 73.20
& 59.33
& 52.19\\

Standard PRM
& 20.86{\small\textcolor{improvegray}{$\pm 0.36$}}
& 53.52{\small\textcolor{improvegray}{$\pm 0.64$}}
& 74.97{\small\textcolor{improvegray}{$\pm 1.04$}}
& 59.61{\small\textcolor{improvegray}{$\pm 0.27$}}
& 52.24 \\

BetaPRM
& 21.14{\small\textcolor{improvegray}{$\pm 0.57$}}
& 53.39{\small\textcolor{improvegray}{$\pm 0.33$}}
& 74.79{\small\textcolor{improvegray}{$\pm 0.70$}}
& 60.07{\small\textcolor{improvegray}{$\pm 0.44$}}
& 52.35 \\

DreamPRM
& 21.06{\small\textcolor{improvegray}{$\pm 0.62$}}
& 53.71{\small\textcolor{improvegray}{$\pm 0.20$}}
& 74.15{\small\textcolor{improvegray}{$\pm 1.22$}}
& 60.27{\small\textcolor{improvegray}{$\pm 0.65$}}
& 52.30\\

\rowcolor{lightpurple}
BayesianPRM
& 21.52{\small\textcolor{improvegray}{$\pm 0.39$}}
& \textbf{54.29}{\small\textcolor{improvegray}{$\pm 0.56$}}
& \textbf{75.71}{\small\textcolor{improvegray}{$\pm 0.78$}}
& \textbf{60.69}{\small\textcolor{improvegray}{$\pm 0.53$}}
& \textbf{53.05}\\

\midrule
\multicolumn{6}{c}{\textit{\textbf{IAS ($N_{\mathrm{max}}=16$)}}} \\
\midrule

Standard PRM
& 20.73{\small\textcolor{improvegray}{$\pm 0.46$}}
& 52.73{\small\textcolor{improvegray}{$\pm 0.49$}}
& 73.81{\small\textcolor{improvegray}{$\pm 0.81$}}
& 59.40{\small\textcolor{improvegray}{$\pm 0.32$}}
& 51.67\\

BetaPRM
& 21.08{\small\textcolor{improvegray}{$\pm 0.71$}}
& 52.63{\small\textcolor{improvegray}{$\pm 0.30$}}
& 73.86{\small\textcolor{improvegray}{$\pm 0.68$}}
& 59.35{\small\textcolor{improvegray}{$\pm 0.43$}}
& 51.73\\

DreamPRM
& 20.89{\small\textcolor{improvegray}{$\pm 1.36$}}
& 53.11{\small\textcolor{improvegray}{$\pm 0.48$}}
& 74.06{\small\textcolor{improvegray}{$\pm 0.74$}}
& 59.28{\small\textcolor{improvegray}{$\pm 0.56$}}
& 51.84\\

\rowcolor{lightpurple}
BayesianPRM
& \textbf{21.46}{\small\textcolor{improvegray}{$\pm 0.72$}}
& \textbf{54.07}{\small\textcolor{improvegray}{$\pm 0.45$}}
& \textbf{74.65}{\small\textcolor{improvegray}{$\pm 0.77$}}
& \textbf{60.20}{\small\textcolor{improvegray}{$\pm 0.48$}}
& \textbf{52.60}\\
\bottomrule
\end{tabular}
}
\label{tab:app_tts}
\end{table*}

\textbf{Experimental Setup.} The training dataset $\mathcal{D}_\text{train}$ and evaluation datasets $\mathcal{D}_\text{test}$ are kept consistent
with those used in Section~\ref{sec:app_accuracy}. We compare the performance of Best-of-$N$ and IAS. For Best-of-$N$, we set $N = 16$. For each $c_t\in\mathcal{D}_\text{test}$, Qwen3VL-8B-Instruct
generates 16 candidate responses, which are scored by different PRMs, and the response with the highest
score is selected. For IAS, we set $N_\text{max} = 16$ and dynamically allocate the rollout budget; We use Standard PRM, BetaPRM, and DreamPRM
as baselines. For Best-of-$N$ , we add an additional Self-Consistency baseline.

\textbf{Results.} As shown in Table~\ref{tab:app_tts}, BayesianPRM consistently improves PRM-guided test-time scaling with Qwen3-VL-8B-Instruct as the rollout policy. Under Best-of-\(N\) with \(N=16\), BayesianPRM achieves the highest average accuracy of 53.05, outperforming Standard PRM and the second-best PRM baseline, BetaPRM, by 0.81 and 0.70 points, respectively. It achieves the best performance on MathVerse-VO, MathVista, and MMStar, while ranking second on MathVision. Under IAS with \(N_{\max}=16\), BayesianPRM achieves the best performance on all four benchmarks and the highest average accuracy of 52.60, exceeding Standard PRM and the second-best method, DreamPRM, by 0.93 and 0.76 points, respectively. These results show that the test-time scaling improvements of BayesianPRM remain consistent when using different rollout backbones.

\subsection{Evaluation on VisualProcessBench}

In this part, we further evaluate BayesianPRM on VisualProcessBench~\citep{wang2026visualprm400k}, a benchmark designed to evaluate the ability of PRMs to identify erroneous steps in multimodal reasoning. It contains 2,866 samples and 26,950 human-annotated step-level correctness labels, covering five source benchmarks: MMMU~\citep{yue2024mmmu}, MathVision, MathVerse, DynaMath~\citep{zou2025dynamath}, and WeMath~\citep{qiao2025we}. Each reasoning step is annotated as correct, incorrect, or neutral. Following prior work, we exclude neutral steps from evaluation. We report macro-F1 for each source benchmark and compute the overall score as micro-F1 over all non-neutral annotated steps pooled across sources. Standard PRM, BetaPRM, DreamPRM, and BayesianPRM are evaluated on VisualProcessBench using models trained on VisualPRM400K-v1.1. This evaluation directly measures the step-level error detection ability of different PRMs, providing a complementary evaluation to the calibration and test-time scaling results. We use a fixed global classification threshold of 0.8 for all PRMs. Steps with scores above 0.8 are classified as correct, and the same threshold is applied uniformly across all five source benchmarks.

\textbf{Results.} As shown in Table~\ref{tab:processbench}, BayesianPRM achieves the highest overall score of 62.23 on VisualProcessBench, outperforming Standard PRM and the second-best method, BetaPRM, by 2.44 and 1.40 points, respectively. BayesianPRM achieves the best performance on four of the five subsets, with particularly clear improvements on MathVision, where it exceeds the second-best method by 3.45 points. On DynaMath, BayesianPRM is only 0.21 points lower than BetaPRM. These results show that BayesianPRM provides more accurate step-level reward predictions across diverse multimodal reasoning tasks.

\begin{table*}[t]
\centering
\caption{
\textbf{Evaluation results on VisualProcessBench.} Each source column reports macro-F1 on the corresponding subset, while Overall denotes the micro-F1 computed over all non-neutral annotated steps. $\pm$ denotes the standard deviation across three random training seeds.
}
\resizebox{\textwidth}{!}{
\begin{tabular}{lcccccc}
\toprule

\textbf{Model}
& \textbf{Overall}
& \textbf{MathVision}
& \textbf{MathVerse}
& \textbf{MMMU}
& \textbf{DynaMath}
& \textbf{WeMath}
\\
\midrule

Standard PRM
& 59.79  {\small\textcolor{improvegray}{$\pm 0.41$}}
& 57.03{\small\textcolor{improvegray}{$\pm 0.16$}}
& 59.65{\small\textcolor{improvegray}{$\pm 0.26$}}
& 62.23{\small\textcolor{improvegray}{$\pm 0.42$}}
& 59.71{\small\textcolor{improvegray}{$\pm 0.24$}}
& 63.13{\small\textcolor{improvegray}{$\pm 0.60$}}
\\

BetaPRM
& 60.83{\small\textcolor{improvegray}{$\pm 0.23$}}
& 57.83{\small\textcolor{improvegray}{$\pm 0.38$}}
& 60.40{\small\textcolor{improvegray}{$\pm 0.44$}}
& 60.59{\small\textcolor{improvegray}{$\pm 0.26$}}
& \textbf{61.33}{\small\textcolor{improvegray}{$\pm 0.62$}}
& 63.47{\small\textcolor{improvegray}{$\pm 0.48$}}
\\

DreamPRM
& 60.69{\small\textcolor{improvegray}{$\pm 0.32$}}
& 58.13{\small\textcolor{improvegray}{$\pm 0.47$}}
& 59.44{\small\textcolor{improvegray}{$\pm 0.29$}}
& 61.22{\small\textcolor{improvegray}{$\pm 0.54$}}
& 61.06{\small\textcolor{improvegray}{$\pm 0.15$}}
& 63.59{\small\textcolor{improvegray}{$\pm 0.27$}}
\\

\rowcolor{lightpurple}
BayesianPRM
& \textbf{62.23}{\small\textcolor{improvegray}{$\pm 0.29$}}
& \textbf{61.58}{\small\textcolor{improvegray}{$\pm 0.54$}}
& \textbf{61.31}{\small\textcolor{improvegray}{$\pm 0.31$}}
& \textbf{63.39}{\small\textcolor{improvegray}{$\pm 0.36$}}
& 61.12{\small\textcolor{improvegray}{$\pm 0.47$}}
& \textbf{64.06}{\small\textcolor{improvegray}{$\pm 0.22$}}
\\
\bottomrule
\end{tabular}
}
\label{tab:processbench}
\end{table*}

\subsection{Evaluation with Calibration Sets}
\label{sec:app_calibration_sets}

In this section, we compare BayesianPRM with additional baselines, including QR-PRM and post-hoc calibration methods including Temperature Scaling, Isotonic Regression, and Histogram Binning, as described in Appendix~\ref{sec:app_baselines}. Since these methods require additional calibration sets, we exclude them from the main-text comparison to ensure a fair evaluation. We also exclude OT-PRM from the baselines because its official implementation has not been publicly released.

\textbf{Experimental Setup.} We randomly split VisualPRM400K-v1.1 into an 80\% training set and a disjoint 20\% calibration set.
The Monte-Carlo labels for the calibration set are obtained using InternVL3-8B with 16 rollouts per sample.
For BayesianPRM, we train the reward ensemble on the 80\% training split and train the belief network on the held-out 20\% calibration split.
For QR-PRM, we initialize from a Standard PRM trained only on the same 80\% training split and fine-tune it on the held-out calibration split following the procedure described in Appendix~\ref{sec:app_baselines}.
For Temperature Scaling, Isotonic Regression, and Histogram Binning, we use the same 80\%-trained Standard PRM as the backbone and fit the calibration mappings exclusively on the held-out calibration split.
All PRMs use InternVL3-8B as the LLM backbone.
Following the setup in Section~\ref{sec:experiments}, we evaluate these methods in terms of both process reward accuracy and test-time scaling performance.
The 80\%-trained Standard PRM is used only for the calibration comparison in this section, while the Standard PRM reported in the main experiments is trained on the full VisualPRM400K-v1.1 dataset.

\textbf{Process Reward Accuracy Evaluation.} As shown in Table~\ref{tab:app_calibration_qr}, BayesianPRM achieves the best results on every benchmark under all four evaluation metrics. Compared with QR-PRM, BayesianPRM reduces the average Brier Score from 0.1668 to 0.1231 and the Positive Brier Score from 0.1508 to 0.1095. Similar improvements are observed in calibration metrics, with the ECE decreasing from 0.1869 to 0.1264 and AverageCE from 0.1728 to 0.1267. In contrast, the three post-hoc calibration methods, including Temperature Scaling, Isotonic Regression, and Histogram Binning, show limited effectiveness in this setting and consistently underperform QR-PRM and BayesianPRM across most metrics. These results show that BayesianPRM remains effective when compared with methods that explicitly leverage an additional calibration set.

\begin{table*}[t]
\centering
\caption{
\textbf{PRM accuracy evaluation results with calibration sets.} We evaluate Brier Score, Positive Brier Score, ECE, and AverageCE on four benchmarks: MathVision, the Vision Only subset of MathVerse (denoted as MathVerse-VO), MathVista, and MMStar. TS, IR, HB denote Temperature Scaling, Isotonic Regression, Histogram Binning, respectively. {\color{improvegray}$\downarrow$} indicates that lower values are better. $\pm$ denotes the standard deviation across three random training seeds.
}
\resizebox{\textwidth}{!}{
\begin{tabular}{lccccc}
\toprule

& \textbf{MathVision}
& \textbf{MathVerse-VO}
& \textbf{MathVista}
& \textbf{MMStar}
& \textbf{Average} \\
\midrule

\multicolumn{6}{c}{\textit{\textbf{Brier Score}} {\color{improvegray}$\downarrow$}} \\
\midrule

QR-PRM
& 0.1634{\small\textcolor{improvegray}{$\pm 0.0166$}}
& 0.1899{\small\textcolor{improvegray}{$\pm 0.0147$}}
& 0.1343{\small\textcolor{improvegray}{$\pm 0.0124$}}
& 0.1795{\small\textcolor{improvegray}{$\pm 0.0262$}}
& 0.1668\\

TS
& 0.2931{\small\textcolor{improvegray}{$\pm 0.0216$}}
& 0.2652{\small\textcolor{improvegray}{$\pm 0.0124$}}
& 0.2145{\small\textcolor{improvegray}{$\pm 0.0189$}}
& 0.2667{\small\textcolor{improvegray}{$\pm 0.0160$}}
& 0.2599\\

IR
& 0.2568{\small\textcolor{improvegray}{$\pm 0.0357$}}
& 0.2420{\small\textcolor{improvegray}{$\pm 0.0324$}}
& 0.2001{\small\textcolor{improvegray}{$\pm 0.0164$}}
& 0.2504{\small\textcolor{improvegray}{$\pm 0.0293$}}
& 0.2373\\

HB
& 0.3146{\small\textcolor{improvegray}{$\pm 0.0250$}}
& 0.2847{\small\textcolor{improvegray}{$\pm 0.0127$}}
& 0.2394{\small\textcolor{improvegray}{$\pm 0.0165$}}
& 0.2781{\small\textcolor{improvegray}{$\pm 0.0277$}}
& 0.2792\\

\rowcolor{lightpurple}
BayesianPRM
& \textbf{0.1035}{\small\textcolor{improvegray}{$\pm 0.0142$}}
& \textbf{0.1422}{\small\textcolor{improvegray}{$\pm 0.0169$}}
& \textbf{0.0946}{\small\textcolor{improvegray}{$\pm 0.0125$}}
& \textbf{0.1521}{\small\textcolor{improvegray}{$\pm 0.0178$}}
& \textbf{0.1231}\\

\midrule
\multicolumn{6}{c}{\textit{\textbf{Positive Brier Score}}
{\color{improvegray}$\downarrow$}} \\
\midrule

QR-PRM
& 0.1548{\small\textcolor{improvegray}{$\pm 0.0247$}}
& 0.1697{\small\textcolor{improvegray}{$\pm 0.0206$}}
& 0.1182{\small\textcolor{improvegray}{$\pm 0.0122$}}
& 0.1603{\small\textcolor{improvegray}{$\pm 0.0281$}}
& 0.1508\\

TS
& 0.2855{\small\textcolor{improvegray}{$\pm 0.0157$}}
& 0.2437{\small\textcolor{improvegray}{$\pm 0.0231$}}
& 0.1964{\small\textcolor{improvegray}{$\pm 0.0152$}}
& 0.2503{\small\textcolor{improvegray}{$\pm 0.0204$}}
& 0.2440\\

IR
& 0.2416{\small\textcolor{improvegray}{$\pm 0.0214$}}
& 0.2309{\small\textcolor{improvegray}{$\pm 0.0195$}}
& 0.1874{\small\textcolor{improvegray}{$\pm 0.0267$}}
& 0.2377{\small\textcolor{improvegray}{$\pm 0.0347$}}
& 0.2244\\

HB
& 0.2988{\small\textcolor{improvegray}{$\pm 0.0273$}}
& 0.2713{\small\textcolor{improvegray}{$\pm 0.0215$}}
& 0.2106{\small\textcolor{improvegray}{$\pm 0.0155$}}
& 0.2622{\small\textcolor{improvegray}{$\pm 0.0303$}}
& 0.2607\\

\rowcolor{lightpurple}
BayesianPRM
& \textbf{0.0914}{\small\textcolor{improvegray}{$\pm 0.0065$}}
& \textbf{0.1337}{\small\textcolor{improvegray}{$\pm 0.0143$}}
& \textbf{0.0735}{\small\textcolor{improvegray}{$\pm 0.0157$}}
& \textbf{0.1394}{\small\textcolor{improvegray}{$\pm 0.0166$}}
& \textbf{0.1095}\\

\midrule
\multicolumn{6}{c}{\textit{\textbf{ECE}}
{\color{improvegray}$\downarrow$}} \\
\midrule

QR-PRM
& 0.2083{\small\textcolor{improvegray}{$\pm 0.0255$}}
& 0.2105{\small\textcolor{improvegray}{$\pm 0.0136$}}
& 0.1442{\small\textcolor{improvegray}{$\pm 0.0246$}}
& 0.1846{\small\textcolor{improvegray}{$\pm 0.0309$}}
& 0.1869\\

TS
& 0.3357{\small\textcolor{improvegray}{$\pm 0.0260$}}
& 0.3102{\small\textcolor{improvegray}{$\pm 0.0391$}}
& 0.2451{\small\textcolor{improvegray}{$\pm 0.0241$}}
& 0.2683{\small\textcolor{improvegray}{$\pm 0.0272$}}
& 0.2898\\

IR
& 0.2943{\small\textcolor{improvegray}{$\pm 0.0323$}}
& 0.3127{\small\textcolor{improvegray}{$\pm 0.0224$}}
& 0.2670{\small\textcolor{improvegray}{$\pm 0.0149$}}
& 0.2463{\small\textcolor{improvegray}{$\pm 0.0266$}}
& 0.2801\\

HB
& 0.3427{\small\textcolor{improvegray}{$\pm 0.0266$}}
& 0.2896{\small\textcolor{improvegray}{$\pm 0.0199$}}
& 0.2613{\small\textcolor{improvegray}{$\pm 0.0163$}}
& 0.2951{\small\textcolor{improvegray}{$\pm 0.0295$}}
& 0.2972\\

\rowcolor{lightpurple}
BayesianPRM
& \textbf{0.1602}{\small\textcolor{improvegray}{$\pm 0.0131$}}
& \textbf{0.1578}{\small\textcolor{improvegray}{$\pm 0.0182$}}
& \textbf{0.0528}{\small\textcolor{improvegray}{$\pm 0.0095$}}
& \textbf{0.1349}{\small\textcolor{improvegray}{$\pm 0.0163$}}
& \textbf{0.1264}\\
\midrule
\multicolumn{6}{c}{\textit{\textbf{AverageCE}}
{\color{improvegray}$\downarrow$}} \\
\midrule

QR-PRM
& 0.2088{\small\textcolor{improvegray}{$\pm 0.0249$}}
& 0.1867{\small\textcolor{improvegray}{$\pm 0.0280$}}
& 0.1261{\small\textcolor{improvegray}{$\pm 0.0156$}}
& 0.1694{\small\textcolor{improvegray}{$\pm 0.0142$}}
& 0.1728\\

TS
& 0.3021{\small\textcolor{improvegray}{$\pm 0.0141$}}
& 0.2752{\small\textcolor{improvegray}{$\pm 0.0256$}}
& 0.2161{\small\textcolor{improvegray}{$\pm 0.0340$}}
& 0.2430{\small\textcolor{improvegray}{$\pm 0.0244$}}
& 0.2591 \\

IR
& 0.2710{\small\textcolor{improvegray}{$\pm 0.0130$}}
& 0.2899{\small\textcolor{improvegray}{$\pm 0.0291$}}
& 0.2351{\small\textcolor{improvegray}{$\pm 0.0201$}}
& 0.2244{\small\textcolor{improvegray}{$\pm 0.0107$}}
& 0.2551\\

HB
& 0.3281{\small\textcolor{improvegray}{$\pm 0.0142$}}
& 0.2640{\small\textcolor{improvegray}{$\pm 0.0190$}}
& 0.2388{\small\textcolor{improvegray}{$\pm 0.0203$}}
& 0.2542{\small\textcolor{improvegray}{$\pm 0.0082$}}
& 0.2713\\

\rowcolor{lightpurple}
BayesianPRM
& \textbf{0.1720}{\small\textcolor{improvegray}{$\pm 0.0135$}}
& \textbf{0.1453}{\small\textcolor{improvegray}{$\pm 0.0146$}}
& \textbf{0.0792}{\small\textcolor{improvegray}{$\pm 0.0042$}}
& \textbf{0.1103}{\small\textcolor{improvegray}{$\pm 0.0060$}}
& \textbf{0.1267}\\
\bottomrule
\end{tabular}
}
\label{tab:app_calibration_qr}
\end{table*}

\textbf{Test-Time Scaling Performance.} As shown in Table~\ref{tab:tts_app}, BayesianPRM achieves the best overall test-time scaling performance under both fixed-budget Best-of-\(N\) and instance-adaptive scaling. With \(N=16\), BayesianPRM obtains the highest average accuracy of 56.75, outperforming QR-PRM by 0.74 points, and achieves the best performance on all four benchmarks. Under IAS with \(N_{\max}=16\), BayesianPRM again achieves the highest average accuracy of 56.14, exceeding QR-PRM by 0.64 points. It performs best on three of the four benchmarks and is only slightly lower than QR-PRM on MMStar. In contrast, the post-hoc calibration methods, including Temperature Scaling, Isotonic Regression, and Histogram Binning, show relatively weaker test-time scaling performance. These results further demonstrate the effectiveness of BayesianPRM even when compared with methods that leverage an additional calibration set.

\begin{table*}[t]
\centering
\caption{
\textbf{PRM-guided test-time scaling performance comparison with calibration sets.} We compare Best-of-\(N\) (\(N=16\)) and Instance-Adaptive Scaling (IAS, \(N_{\max}=16\)) on MathVision, MathVerse-VO, MathVista and MMStar.  {\color{improvegray}$\uparrow$} indicates that higher values are better. $\pm$ denotes the standard deviation across three random training seeds.
}
\resizebox{\textwidth}{!}{
\begin{tabular}{lccccc}
\toprule

& \textbf{MathVision}
& \textbf{MathVerse-VO}
& \textbf{MathVista}
& \textbf{MMStar}
& \textbf{Average} \\

\midrule
\multicolumn{6}{c}{\textit{\textbf{Best-of-$N$ ($N=16$)}}{\color{improvegray}$\uparrow$}} \\
\midrule

QR-PRM
& 34.79{\small\textcolor{improvegray}{$\pm 0.67$}}
& 51.26{\small\textcolor{improvegray}{$\pm 0.84$}}
& 69.28{\small\textcolor{improvegray}{$\pm 0.73$}}
& 68.71{\small\textcolor{improvegray}{$\pm 0.40$}}
& 56.01 \\

TS
& 32.10{\small\textcolor{improvegray}{$\pm 1.18$}}
& 49.66{\small\textcolor{improvegray}{$\pm 0.64$}}
& 69.12{\small\textcolor{improvegray}{$\pm 0.48$}}
& 67.01{\small\textcolor{improvegray}{$\pm 0.39$}}
& 54.47\\

IR
& 32.67{\small\textcolor{improvegray}{$\pm 1.14$}}
& 51.26{\small\textcolor{improvegray}{$\pm 0.79$}}
& 68.83{\small\textcolor{improvegray}{$\pm 0.85$}}
& 67.58{\small\textcolor{improvegray}{$\pm 0.30$}}
& 55.09\\

HB
& 32.00{\small\textcolor{improvegray}{$\pm 1.01$}}
& 49.13{\small\textcolor{improvegray}{$\pm 0.57$}}
& 69.06{\small\textcolor{improvegray}{$\pm 0.82$}}
& 67.24{\small\textcolor{improvegray}{$\pm 1.13$}}
& 54.36\\

\rowcolor{lightpurple}
BayesianPRM
& \textbf{35.80}{\small\textcolor{improvegray}{$\pm 0.56$}}
& \textbf{51.93}{\small\textcolor{improvegray}{$\pm 0.77$}}
& \textbf{69.56}{\small\textcolor{improvegray}{$\pm 0.33$}}
& \textbf{69.69}{\small\textcolor{improvegray}{$\pm 0.90$}}
& \textbf{56.75}\\

\midrule
\multicolumn{6}{c}{
    \textbf{\textit{IAS ($N_{\max}=16$)}}
    {\color{improvegray}{$\uparrow$}}
} \\
\midrule

QR-PRM
& 34.25{\small\textcolor{improvegray}{$\pm 0.83$}}
& 50.86{\small\textcolor{improvegray}{$\pm 0.45$}}
& 68.24{\small\textcolor{improvegray}{$\pm 0.36$}}
& \textbf{68.66}{\small\textcolor{improvegray}{$\pm 0.69$}}
& 55.50\\

TS
& 32.01{\small\textcolor{improvegray}{$\pm 1.20$}}
& 49.12{\small\textcolor{improvegray}{$\pm 0.93$}}
& 68.74{\small\textcolor{improvegray}{$\pm 0.65$}}
& 67.35{\small\textcolor{improvegray}{$\pm 0.46$}}
& 54.31\\

IR
& 32.33{\small\textcolor{improvegray}{$\pm 0.96$}}
& 51.19{\small\textcolor{improvegray}{$\pm 0.42$}}
& 68.41{\small\textcolor{improvegray}{$\pm 0.79$}}
& 67.20{\small\textcolor{improvegray}{$\pm 0.63$}}
& 54.78\\

HB
& 31.84{\small\textcolor{improvegray}{$\pm 0.35$}}
& 49.06{\small\textcolor{improvegray}{$\pm 1.33$}}
& 68.51{\small\textcolor{improvegray}{$\pm 0.58$}}
& 67.19{\small\textcolor{improvegray}{$\pm 0.75$}}
& 54.15\\

\rowcolor{lightpurple}
BayesianPRM
& \textbf{35.33}{\small\textcolor{improvegray}{$\pm 0.46$}}
& \textbf{51.61}{\small\textcolor{improvegray}{$\pm 0.74$}}
& \textbf{69.34}{\small\textcolor{improvegray}{$\pm 0.71$}}
& 68.29{\small\textcolor{improvegray}{$\pm 0.34$}}
& \textbf{56.14}\\
\bottomrule
\end{tabular}
}
\label{tab:tts_app}
\end{table*}

\subsection{Extended Ablation Study}
\label{sec:app_ablation}
In this section, we provide additional ablation study results omitted from the main text due to space limits.

\textbf{Comparison with Vanilla Ensemble PRM.} We construct a vanilla ensemble PRM by uniformly averaging the predictions of all ensemble members, without performing belief updating or calibration. We compare its test-time scaling performance with BayesianPRM, with the results reported in Table~\ref{tab:ablation_vanilla}. It is clear that BayesianPRM consistently outperforms vanilla ensemble PRM, demonstrating the importance of reliability-aware belief updating and conservatism-aware belief calibration.

\begin{table*}[t]
\centering
\caption{
\textbf{Comparison with Vanilla Ensemble PRM}. We compare Best-of-\(N\) (\(N=16\)) on MathVision, MathVerse-VO, MathVista and MMStar. $\pm$ denotes the standard deviation across three random training seeds.
}
\resizebox{\textwidth}{!}{
\begin{tabular}{lccccc}
\toprule

& \textbf{MathVision}
& \textbf{MathVerse-VO}
& \textbf{MathVista}
& \textbf{MMStar}
& \textbf{Average} \\

\midrule

Vanilla Ensemble PRM
& 34.88{\small\textcolor{improvegray}{$\pm 0.47$}}
& 50.21{\small\textcolor{improvegray}{$\pm 0.92$}}
& 68.26{\small\textcolor{improvegray}{$\pm 0.59$}}
& 68.03{\small\textcolor{improvegray}{$\pm 0.88$}}
& 55.35 \\

\rowcolor{lightpurple}
BayesianPRM
& \textbf{35.86}{\small\textcolor{improvegray}{$\pm 0.70$}}
& \textbf{51.40}{\small\textcolor{improvegray}{$\pm 1.11$}}
& \textbf{69.40}{\small\textcolor{improvegray}{$\pm 0.53$}}
& \textbf{69.33}{\small\textcolor{improvegray}{$\pm 0.60$}}
& \textbf{56.50}\\

\bottomrule
\end{tabular}
}
\label{tab:ablation_vanilla}
\end{table*}

\textbf{Effect of $\beta_1$.} $\beta_1$ controls the strength of reliability-aware belief updating, with a smaller \(\beta_1\) inducing a stronger update. If \(\beta_1\) is too large, the belief distribution remains overly close to the uniform prior. In contrast, an overly small \(\beta_1\) produces an excessively sharp belief distribution, which may overfit to label noise. We vary \(\beta_1\) over \(\{0.001, 0.01, 0.05, 0.1\}\) and report the test-time scaling performance in Table~\ref{tab:ablation_app_beta1}. As shown, the performance drops when \(\beta_1\) is set to either 0.001 or 0.1. We also plot in Figure~\ref{fig:weight} how the reliability-aware belief over reward heads evolves during training for different values of \(\beta_1\), with the reward ensemble size fixed to 5. When \(\beta_1=0.001\) or \(0.01\), Head 2 tends to dominate while the weights of Head 4 and Head 5 approach zero, indicating an overly concentrated belief distribution. In contrast, when \(\beta_1=0.1\), the belief remains too evenly distributed across the heads. Therefore, we set $\beta_1=0.05$ in our experiments.

\begin{figure}
    \centering
    \includegraphics[width=0.98\linewidth]{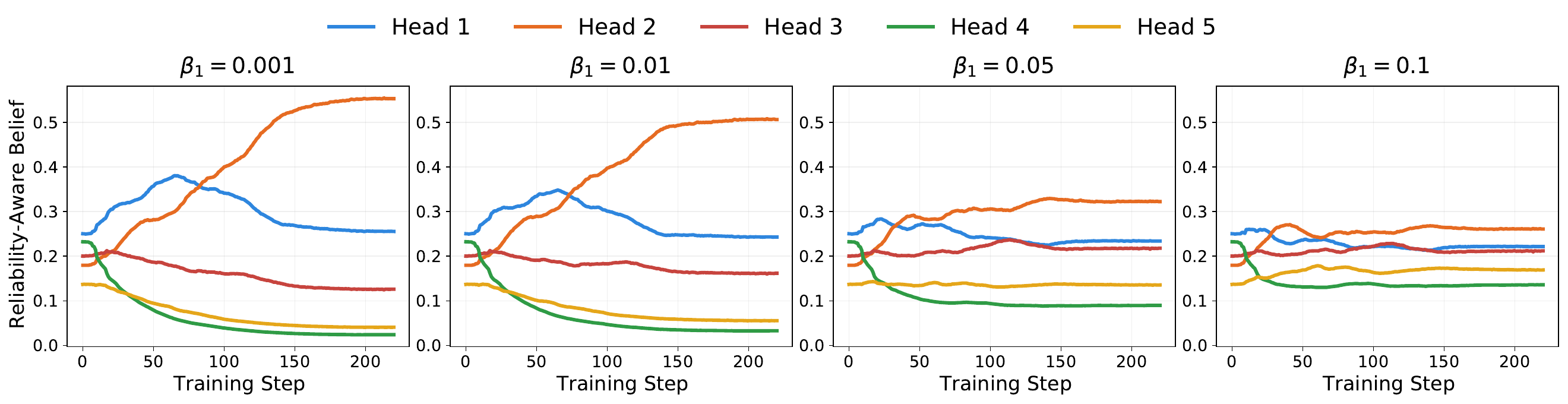}
    \caption{Reliability-aware belief distribution during training. We use $M=5$ for visualization clarity.}
    \label{fig:weight}
\end{figure}

\begin{table*}[t]
\centering
\caption{
\textbf{Ablation Study on $\beta_1$}. We compare Best-of-\(N\) (\(N=16\)) on MathVision, MathVerse-VO, MathVista and MMStar. $\pm$ denotes the standard deviation across three random training seeds.
}
\resizebox{\textwidth}{!}{
\begin{tabular}{lccccc}
\toprule

& \textbf{MathVision}
& \textbf{MathVerse-VO}
& \textbf{MathVista}
& \textbf{MMStar}
& \textbf{Average} \\

\midrule

$\beta_1=0.001$
& 34.28{\small\textcolor{improvegray}{$\pm 0.46$}}
& 50.30{\small\textcolor{improvegray}{$\pm 0.88$}}
& 68.33{\small\textcolor{improvegray}{$\pm 0.67$}}
& 68.56{\small\textcolor{improvegray}{$\pm 0.83$}}
& 55.37 \\

$\beta_1=0.01$
& 35.67{\small\textcolor{improvegray}{$\pm 0.78$}}
& 51.86{\small\textcolor{improvegray}{$\pm 1.27$}}
& 69.32{\small\textcolor{improvegray}{$\pm 0.42$}}
& 69.07{\small\textcolor{improvegray}{$\pm 0.57$}}
& 56.48\\

$\beta_1=0.05$
& 35.86{\small\textcolor{improvegray}{$\pm 0.70$}}
& 51.40{\small\textcolor{improvegray}{$\pm 1.11$}}
& 69.40{\small\textcolor{improvegray}{$\pm 0.53$}}
& 69.33{\small\textcolor{improvegray}{$\pm 0.60$}}
& 56.50\\

$\beta_1=0.1$
& 35.21{\small\textcolor{improvegray}{$\pm 0.40$}}
& 50.35{\small\textcolor{improvegray}{$\pm 0.63$}}
& 68.82{\small\textcolor{improvegray}{$\pm 0.27$}}
& 68.19{\small\textcolor{improvegray}{$\pm 0.64$}}
& 55.64\\

\bottomrule
\end{tabular}
}
\label{tab:ablation_app_beta1}
\end{table*}

\begin{table*}[t]
\centering
\caption{
\textbf{Ablation Study on $\beta_2$}. We compare Best-of-\(N\) (\(N=16\)) on MathVision, MathVerse-VO, MathVista and MMStar. $\pm$ denotes the standard deviation across three random training seeds.
}
\resizebox{\textwidth}{!}{
\begin{tabular}{lccccc}
\toprule

& \textbf{MathVision}
& \textbf{MathVerse-VO}
& \textbf{MathVista}
& \textbf{MMStar}
& \textbf{Average} \\

\midrule

$\beta_2=0.001$
& 35.24{\small\textcolor{improvegray}{$\pm 0.62$}}
& 50.62{\small\textcolor{improvegray}{$\pm 0.94$}}
& 68.43{\small\textcolor{improvegray}{$\pm 0.46$}}
& 69.10{\small\textcolor{improvegray}{$\pm 0.82$}}
& 55.85 \\

$\beta_2=0.01$
& 35.78{\small\textcolor{improvegray}{$\pm 0.76$}}
& 50.54{\small\textcolor{improvegray}{$\pm 0.85$}}
& 69.21{\small\textcolor{improvegray}{$\pm 0.50$}}
& 69.39{\small\textcolor{improvegray}{$\pm 0.74$}}
& 56.23\\

$\beta_2=0.05$
& 35.86{\small\textcolor{improvegray}{$\pm 0.70$}}
& 51.40{\small\textcolor{improvegray}{$\pm 1.11$}}
& 69.40{\small\textcolor{improvegray}{$\pm 0.53$}}
& 69.33{\small\textcolor{improvegray}{$\pm 0.60$}}
& 56.50\\

$\beta_2=0.1$
& 34.72{\small\textcolor{improvegray}{$\pm 0.49$}}
& 50.26{\small\textcolor{improvegray}{$\pm 1.26$}}
& 69.27{\small\textcolor{improvegray}{$\pm 0.61$}}
& 69.98{\small\textcolor{improvegray}{$\pm 0.85$}}
& 56.06\\

\bottomrule
\end{tabular}
}
\label{tab:ablation_app_beta2}
\end{table*}

\textbf{Effect of $\beta_2$.} $\beta_2$ controls the strength of conservatism-aware belief calibration, with a smaller \(\beta_2\) inducing stronger conservative reweighting. If \(\beta_2\) is too large, reward overestimation may not be suppressed, whereas an overly small \(\beta_2\) may instead lead to reward underestimation. We vary \(\beta_2\) over \(\{0.001, 0.01, 0.05, 0.1\}\), and Table~\ref{tab:ablation_app_beta2} reports the test-time scaling performance comparison. As shown, performance degrades when \(\beta_2\) is either too large or too small. We also observe that larger \(\beta_2\) values are more favorable on MMStar, suggesting that the degree of reward overestimation varies across datasets. Therefore, a single \(\beta_2\) may not be optimal for all benchmarks. In our experiments, we uniformly set \(\beta_2=0.05\) unless otherwise specified, rather than extensively tuning it for each benchmark.

\begin{table*}[t]
\centering
\caption{
\textbf{Ablation Study on $M$}. We compare Best-of-\(N\) (\(N=16\)) on MathVision, MathVerse-VO, MathVista and MMStar. $\pm$ denotes the standard deviation across three random training seeds.
}
\resizebox{\textwidth}{!}{
\begin{tabular}{lccccc}
\toprule

& \textbf{MathVision}
& \textbf{MathVerse-VO}
& \textbf{MathVista}
& \textbf{MMStar}
& \textbf{Average} \\

\midrule

$M=10$
& 35.30{\small\textcolor{improvegray}{$\pm 0.72$}}
& 51.42{\small\textcolor{improvegray}{$\pm 0.79$}}
& 68.97{\small\textcolor{improvegray}{$\pm 0.42$}}
& 68.49{\small\textcolor{improvegray}{$\pm 0.59$}}
& 56.05\\

$M=20$
& 35.86{\small\textcolor{improvegray}{$\pm 0.70$}}
& 51.40{\small\textcolor{improvegray}{$\pm 1.11$}}
& 69.40{\small\textcolor{improvegray}{$\pm 0.53$}}
& 69.33{\small\textcolor{improvegray}{$\pm 0.60$}}
& 56.50\\

$M=50$
& 35.89{\small\textcolor{improvegray}{$\pm 0.65$}}
& 51.35{\small\textcolor{improvegray}{$\pm 1.30$}}
& 69.52{\small\textcolor{improvegray}{$\pm 0.49$}}
& 68.99{\small\textcolor{improvegray}{$\pm 0.64$}}
& 56.44\\

$M=100$
& 35.87{\small\textcolor{improvegray}{$\pm 0.58$}}
& 51.46{\small\textcolor{improvegray}{$\pm 0.83$}}
& 69.34{\small\textcolor{improvegray}{$\pm 0.60$}}
& 69.62{\small\textcolor{improvegray}{$\pm 0.39$}}
& 56.57\\

\bottomrule
\end{tabular}
}
\label{tab:ablation_app_m}
\end{table*}

\begin{table*}
\centering
\caption{
\textbf{Ablation Study on} $\lambda_\text{prior}$. We compare Best-of-\(N\) (\(N=16\)) on MathVision, MathVerse-VO, MathVista and MMStar. $\pm$ denotes the standard deviation across three random training seeds.
}
\resizebox{\textwidth}{!}{
\begin{tabular}{lccccc}
\toprule

& \textbf{MathVision}
& \textbf{MathVerse-VO}
& \textbf{MathVista}
& \textbf{MMStar}
& \textbf{Average} \\

\midrule

$\lambda_\text{prior}=0.1$
& 34.93{\small\textcolor{improvegray}{$\pm 0.49$}}
& 49.84{\small\textcolor{improvegray}{$\pm 0.53$}}
& 68.20{\small\textcolor{improvegray}{$\pm 0.73$}}
& 67.77{\small\textcolor{improvegray}{$\pm 0.43$}}
& 55.19\\

$\lambda_\text{prior}=1$
& 35.39{\small\textcolor{improvegray}{$\pm 0.79$}}
& 50.27{\small\textcolor{improvegray}{$\pm 0.88$}}
& 68.85{\small\textcolor{improvegray}{$\pm 0.56$}}
& 69.12{\small\textcolor{improvegray}{$\pm 0.93$}}
& 55.91\\

$\lambda_\text{prior}=10$
& 35.86{\small\textcolor{improvegray}{$\pm 0.70$}}
& 51.40{\small\textcolor{improvegray}{$\pm 1.11$}}
& 69.40{\small\textcolor{improvegray}{$\pm 0.53$}}
& 69.33{\small\textcolor{improvegray}{$\pm 0.60$}}
& 56.50\\

$\lambda_\text{prior}=50$
& 35.41{\small\textcolor{improvegray}{$\pm 1.09$}}
& 51.02{\small\textcolor{improvegray}{$\pm 0.69$}}
& 68.82{\small\textcolor{improvegray}{$\pm 0.64$}}
& 69.24{\small\textcolor{improvegray}{$\pm 0.48$}}
& 56.12\\

\bottomrule
\end{tabular}
}
\label{tab:ablation_app_prior}
\end{table*}

\textbf{Effect of $M$.} The ensemble size \(M\) determines the number of reward heads. When \(M\) is too small, the reward hypothesis space may not be sufficiently covered, whereas an excessively large \(M\) may introduce redundant reward heads. We vary \(M\) over \(\{10,20,50,100\}\), and report the test-time scaling performance in Table~\ref{tab:ablation_app_m}. We observe that the performance starts to fluctuate when \(M>20\). Therefore, we set \(M=20\) in our experiments.

\textbf{Effect of} $\lambda_\text{prior}$\textbf{.} $\lambda_\text{prior}$ controls the contribution of the random prior network to each reward head. When $\lambda_\text{prior}$ is too small, the reward heads lack sufficient diversity. However, an overly large $\lambda_\text{prior}$ may cause the random prior to dominate the learned reward prediction and degrade reward accuracy. We vary \(\lambda_{\text{prior}}\) over \(\{0.1,1,10,50\}\). As shown in Table~\ref{tab:ablation_app_prior}, the test-time scaling performance degrades when \(\lambda_{\text{prior}}\) is either too small or too large. Therefore, we set \(\lambda_{\text{prior}}=10\) throughout our experiments.

\section{Computational Overhead}

BayesianPRM introduces only a small parameter overhead over the standard PRM. The standard PRM contains 7.944B parameters, while BayesianPRM contains 7.982B parameters, corresponding to a 0.48\% parameter overhead. 

We further compare the training cost on 4 \(\times\) H100 GPUs. Training the Standard PRM for one epoch takes approximately 24 hours, whereas the complete training pipeline of BayesianPRM requires less than 20 hours. This difference mainly arises because BayesianPRM trains the reward ensemble on 80\% of the data, while the remaining 20\% is used to train the lightweight belief network with the backbone frozen. Therefore, in our training setup, BayesianPRM requires less wall-clock training time than Standard PRM.

\section{Case Study}

In this section, we provide a case study of the step-wise rewards assigned by BayesianPRM, as shown in Figure~\ref{fig:case_study}. Figure~\ref{fig:case_study}   (a) presents a correct reasoning trajectory, where BayesianPRM assigns consistently high rewards to all reasoning steps. Figure~\ref{fig:case_study} (b) presents an incorrect trajectory. At the first erroneous step (Step 8), the reward exhibits a sharp drop from 0.977 to 0.041. These examples illustrate that BayesianPRM can assign high rewards to valid reasoning while identifying erroneous intermediate steps with substantially lower rewards.

\begin{figure}[t]
\centering
\setlength{\tabcolsep}{3pt}

\begin{tcolorbox}[
    colback=white,
    colframe=casegreen,
    boxrule=0.9pt,
    arc=0pt,
    left=3pt,
    right=3pt,
    top=3pt,
    bottom=3pt
]

\begin{minipage}[t]{0.3\linewidth}
\vspace{0pt}
    \centering
    \includegraphics[width=\linewidth]{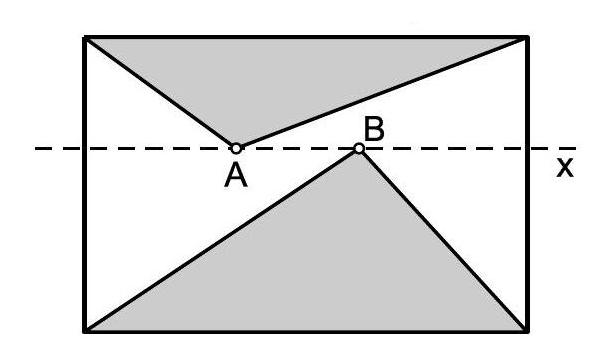}
\end{minipage}
\hfill
\begin{minipage}[t]{0.7\linewidth}
\vspace{5pt}
    \scriptsize
    \textbf{Question:}
    The sum of the areas of the two shaded triangles is
    $10\,\mathrm{cm}^2$. What is the area of the rectangle?

    \vspace{2pt}
    \textbf{Choices:}
    (A) 18 \quad
    (B) 20 \quad
    (C) 22 \quad
    (D) 24

    \vspace{2pt}
    \textbf{Ground Truth:} B

    \textbf{Benchmark:} MathVision
\end{minipage}

\vspace{4pt}

\scriptsize
\begin{tabularx}{\linewidth}{@{}Xr@{}}
\textbf{Step 1:}
Let the rectangle have width $W$ and height $H$.
&
\prmgreen{0.895}
\\[2pt]

\textbf{Step 2:}
Let $h_1$ and $h_2$ be the distances from line $x$
to the top and bottom sides, so $h_1+h_2=H$.
&
\prmgreen{0.906}
\\[2pt]

\textbf{Step 3:}
The upper triangle has area $\frac{1}{2}Wh_1$.
&
\prmgreen{0.944}
\\[2pt]

\textbf{Step 4:}
The lower triangle has area $\frac{1}{2}Wh_2$.
&
\prmgreen{0.968}
\\[2pt]

\textbf{Step 5:}
Their total area is
$\frac{1}{2}W(h_1+h_2)=\frac{1}{2}WH$.
&
\prmgreen{0.968}
\\[2pt]

\textbf{Step 6:}
Since the shaded area is $10$,
$\frac{1}{2}WH=10$.
&
\prmgreen{0.958}
\\[2pt]

\textbf{Step 7:}
Therefore, $WH=20$.
&
\prmgreen{0.992}
\\[2pt]

\textbf{Step 8:}
The rectangle has area $20\,\mathrm{cm}^2$.
&
\prmgreen{0.986}
\\[2pt]

\textcolor{casegreen}{\textbf{Final Answer: B. $20\,\mathrm{cm}^2$ \checkmark}}
\end{tabularx}

\end{tcolorbox}
\vspace{2pt}
\centerline{\textbf{(a)} A correct case.}
\vspace{4pt}

\begin{tcolorbox}[
    colback=white,
    colframe=casered,
    boxrule=0.9pt,
    arc=0pt,
    left=3pt,
    right=3pt,
    top=3pt,
    bottom=3pt
]

\begin{minipage}[t]{0.30\linewidth}
\vspace{0pt}
    \centering
    \includegraphics[width=\linewidth]{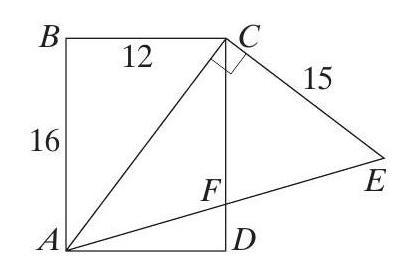}
\end{minipage}
\hfill
\begin{minipage}[t]{0.66\linewidth}
\vspace{5pt}
    \scriptsize
    \textbf{Question:}
    In rectangle $ABCD$, $AB=16$, $BC=12$, $AC\perp CE$, and $CE=15$.
    If $F=AE\cap CD$, find the area of triangle $ACF$.

    \vspace{2pt}
    \textbf{Ground Truth:} $75$

    \textbf{Benchmark:} MathVision
\end{minipage}

\vspace{4pt}

\scriptsize
\begin{tabularx}{\linewidth}{@{}Xr@{}}

\textbf{Step 1:}
Identify the geometric setting: rectangle $ABCD$ with
$AB=16$, $BC=12$, $AC\perp CE$, $CE=15$, and
$F=AE\cap CD$.
&
\prmgreen{0.785}
\\[2pt]

\textbf{Step 2:}
Compute the diagonal:
$AC=\sqrt{16^2+12^2}=20$.
&
\prmgreen{0.865}
\\[2pt]

\textbf{Step 3:}
Since $AC\perp CE$ and $CE=15$, triangle $ACE$ is
a right triangle.
&
\prmgreen{0.922}
\\[2pt]

\textbf{Step 4:}
Set coordinates
$A=(0,0)$, $B=(0,16)$, $C=(12,16)$, and $D=(12,0)$.
&
\prmgreen{0.943}
\\[2pt]

\textbf{Step 5:}
Because the slope of $AC$ is $\frac{4}{3}$,
the slope of $CE$ is $-\frac{3}{4}$.
Using the perpendicular condition and $CE=15$,
solve for $E$.
&
\prmgreen{0.960}
\\[2pt]

\textbf{Step 6:}
This gives two candidates, and the one consistent with the figure is
$E=(24,7)$.
&
\prmgreen{0.913}
\\[2pt]

\textbf{Step 7:}
Hence line $AE$ has equation
$y=\frac{7}{24}x$.
&
\prmgreen{0.977}
\\[2pt]

\textcolor{casered}{
\textbf{Step 8:}
Since $CD$ is the vertical line $x=12$, substitute into
$y=\frac{7}{24}x$ and obtain
$y=\frac{7}{24}\times 12 = 7$.
Thus, $F=(12,7)$.
}
&
\prmred{0.041}
\\[2pt]

\textcolor{casered}{\textbf{Step 9:}
Using $A=(0,0)$, $C=(12,16)$, and $F=(12,7)$,
compute the area of triangle $ACF$ to be
$54$.}
&
\prmred{0.017}
\\[2pt]

\textcolor{casered}{
\textbf{Final Answer: $54$ $\times$}
}

\end{tabularx}

\end{tcolorbox}
\vspace{2pt}
\centerline{\textbf{(b)} An incorrect case.}

\caption{
A case study of BayesianPRM's step-wise evaluation.
}
\label{fig:case_study}

\end{figure}

\section{Limitations}

One limitation of our method is that several important hyperparameters, such as \(\beta_2\), need to be manually specified. However, the degree of reward overestimation can vary across datasets and test-time rollout policies, leading to different optimal hyperparameter values. Moreover, the current conservatism-aware belief calibration performs a one-sided correction toward lower rewards and is therefore most effective when the prediction error is dominated by overestimation. Developing adaptive methods that automatically determine both the strength and direction of calibration would be an interesting direction for future work.

\end{document}